\documentclass[12pt, journal, draftclsnofoot, onecolumn]{IEEEtran}
\usepackage[american]{babel}
\usepackage{algorithm}
\usepackage{algorithmic}
\usepackage{natbib} %
\usepackage{mathtools} %
\usepackage{booktabs} %
\usepackage{tikz} %
\usepackage{bm}
\usepackage{amsmath,amssymb,amsthm,amsfonts,mathrsfs}
\usepackage{subfigure}
\usepackage{enumitem} 
\newtheorem{definition}{Definition}
\newtheorem{theorem}{Theorem}

\newtheorem{lemma}{Lemma}

\newcommand{\NAM}{\textsc{dist-UCRL}}
\begin{document}

\title{Communication Efficient Parallel Reinforcement Learning}

\author{Mridul Agarwal, Bhargav Ganguly, and Vaneet Aggarwal
\\\{agarw180,bganguly,vaneet\}@purdue.edu
\\Purdue University, West Lafayette IN 47907}  %

\maketitle

\begin{abstract}
  We consider the problem where $M$ agents interact with $M$ identical and independent environments with $S$ states and $A$ actions using reinforcement learning for $T$ rounds. The agents share their data with a central server to minimize their regret. We aim to find an algorithm that allows the agents to minimize the regret with infrequent communication rounds. We provide \NAM\ which runs at each agent and prove that the total cumulative regret of $M$ agents is upper bounded as $\Tilde{O}(DS\sqrt{MAT})$ for a Markov Decision Process with diameter $D$, number of states $S$, and number of actions $A$. The agents synchronize after their visitations to any state-action pair exceeds a certain threshold. Using this, we obtain a bound of $O\left(MSA\log(MT)\right)$ on the total number of communications rounds. Finally, we evaluate the algorithm against multiple environments and demonstrate that the proposed algorithm performs at par with an always communication version of the UCRL2 algorithm, while with significantly lower communication.
\end{abstract}

\section{Introduction}\label{sec:intro}

Reinforcement Learning (RL) is being increasingly deployed and trained with parallel agents. In many cases, each agent interacts with identical and independent environments. For example, in autonomous cars, the agents do not share the environment as they are located possibly far away \citep{kiran2020deep}. Similarly, in ride/freight sharing services, different RL agents (vehicles) make decisions in parallel to minimize their costs and maximize the profits \citep{al2019deeppool,manchella2021flexpool}. Further, consider an example of an e-commerce company using RL agents in servers for recommending products to customer. Based on the location of the customer, each customer's query may potentially be routed to a particular server (or agent) \citep{sankararaman2019social}. Finally in the field of robotics, parallel robots are often deployed in practice \citep{hu2020voronoi,sartoretti2019distributed}. From these examples, we note that every agent would gain  by sharing their data collected, otherwise they would be wasting the knowledge accumulated by the remaining agents. Such parallel policy learning with limited amount of communication is the focus of this paper. 

Sharing of data across agents poses a new set of problems which comes from communicating the samples. If the agents communicate the observation tuple (state, action, reward, next state) at every time step, their communication complexity will increase. For various power constrained devices such as small robots, this luxury might not always be available \citep{sankararaman2019social}. In this paper, we aim to work on the problem to reduce the number of communication steps while obtaining the same regret bounds as that of an strategy in which the agents always communicate.

For a system with $M$ parallel agents, the parallel agents generate $M$ times more data. From the lens of regret analysis, this setup loosely translates to a setup where a single agent works for time horizon of $MT$. For the setup where a single agent runs for $MT$ steps, the regret is lower bounded by $O(\sqrt{DSMAT}$ and upper bounded by $\Tilde{O}(DS\sqrt{MAT})$ \citep{jaksch2010near,agrawal2017optimistic}. We show that, if the agents take their actions sequentially and communicate after every interaction with their respective environments, then the agents can collectively obtain a regret bound of $\Tilde{O}(DS\sqrt{MAT})$.
This results in a faster convergence rate of $O(1/\sqrt{MT})$ to the optimal policy as compared to the convergence rate of $O(1/\sqrt{T})$ when using one RL agent. Note that in practice, a sequential decision making for parallel independent agents cannot be guaranteed. Thus, this result acts as a lower bound to parallel reinforcement learning. In this paper, we further provide an algorithm, in which agents work in parallel, that achieves these bounds with limited communication. 

We consider a setup where $M$ reinforcement learning agents interact with $M$ identical environments or Markov Decision Processes (MDPs) \citep{sutton2018reinforcement,puterman1994markov}. Similar to the examples considered above, the $M$ environments are independent. We assume that there also exists a central coordinator (server), although we suggest a method to relax this assumption. All the agents report their experiences to the central coordinator which then computes and shares a policy with the agents. We consider that the central coordinator uses a model based algorithm to compute the policy and each agent stores the number of visitations made to any particular (state, action, next state) tuple. The central coordinator also shares the total visitations to each state-action pair back to all the agents.

Based on this setup, we  provide a novel communication efficient \NAM\ algorithm. In \NAM\ algorithm, the agents communicate with the central coordinator whenever any agent visits any state action pair $s,a$ in the current epoch (since last synchronization) a fraction ($1/M$)   of the total visitation counts in $s,a$ till last synchronization, instead of every $N\geq 1$ time step. Using this synchronization strategy, for an MDP with diameter $D$, number of states $S$, and number of actions $A$, using \NAM\ algorithm, we bound the cumulative regret of $M$ agents for $T$ time horizon scales as $O(DS\sqrt{MAT})$ using only $O(MAS\log_2(MT))$ synchronization steps\footnote{We use the term communication round and synchronization step interchangeably.}. Note that \NAM\ not only achieves the lower bound regret scaling ($\Tilde{O}(\sqrt{MT})$) but also with limited communication. %

To obtain our results, we also derive a concentration bound on $M$ independent Martingale sequences which can be of independent interest. To the best of our knowledge, this is the first work in the direction of obtaining the performance guarantees using regret analysis for a setup where $M$ agents interact and collaborate.%

We also evaluate our algorithms empirically. We run the proposed \NAM\ algorithm in multiple environments. We compare the \NAM\ algorithm against the modified UCRL algorithm, following which the agents communicates after every time step.  We show that \NAM\ algorithm obtain similar regret bound  with reduced communication.

The rest of the paper is organized as follows. Section \ref{sec:related_works} summarizes the key related works. Section \ref{sec:sys_model} describes the complete system model. Section \ref{sec:algo_description} describes the \NAM\ algorithm, and the regret guarantees of the algorithm are provided in Section \ref{sec:main_results}. Section \ref{sec:mod-uc} describes a modification of UCRL2, \textsc{mod-UCRL2}, and its regret guarantees. Section \ref{sec:evaluations} tests the proposed algorithms in multiple environments, and Section \ref{sec:concl} concludes the paper with some possible directions for future works. 

\section{Related Works}\label{sec:related_works}
Optimal planning using Markov Decision Processes has seen numerous significant contributions in history. Many algorithms have been proposed to find the optimal policies. The major algorithms include Q-learning, and policy iteration to find optimal policies \citep{howard1960dynamic,puterman1994markov,bertsekas1995dynamic,sutton2018reinforcement}. Fundamentally, the algorithms work by calculating the utility of taking certain actions in states that maximizes the utility of the state. These algorithms provide optimal policy when the transition probabilities of the Markov Decision Process is known or an $\epsilon$-optimal policy if using iterative algorithms \citep{puterman1994markov}.

A large body of work is also done in the setup where the transition probabilities are not known. The model-based algorithms work by reducing the number samples required to obtain a close estimate of transition probabilities \citep{bartlett2009regal,jaksch2010near}. The model-free algorithms work by directly estimating the utilities obtained by taking an action in a state \citep{jin2018q}. These algorithms apply optimism in the face of uncertainty principle to find a model near the empirical estimates that provides the highest reward.  There are also algorithms that sample the transition probabilities using posterior sampling and obtain regret bound \citep{ian2013more,agrawal2017optimistic}. The algorithms suggest that using parallel agents interacting independent and identical environments will provide tighter concentration inequalities and hence will help in reducing the regret.

In the domain of multiple agents using RL, most of the work is done where the agents interact with a dependent environment and the decision of one agent impacts all the other agents. This area is known as Multi-Agent Reinforcement Learning (MARL) \citep{gupta2017cooperative,zhang2018fully}. The agents may cooperate with each other, for example, in the case of autonomous vehicles yielding on a busy road. Or the agents may also compete with each other, for example, in the case of car racing where only one car may win. Although, we do have multiple agents in our setup, the environments are independent. Hence, we will refrain from using MARL terminology in this paper.

With the introduction of deep learning to the field of reinforcement learning, many ``Deep RL'' algorithms have been provided to minimize the sample complexity to find the optimal policies \citep{mnih2013playing,schulman2015trust,mnih2016asynchronous,schulman2017proximal,haarnoja2018soft}. Recently, various algorithms are using parallel actors to learn better policies using deep reinforcement learning \citep{nair2015massively,clemente2017efficient,horgan2018distributed,espeholt2018impala,assran2019gossip}. These algorithms consist of parallel agents that share the entire sequence of (state, action, reward, next state) tuples after every $n$ epochs, and update a common neural network with the gradients computed from possibly parallel learners. 
Compared to these works, we consider a model based setup to obtain regret guarantees with only $O(MAS\log_2(MT))$ number of synchronization rounds with a common controller learning the policy.

Recently, in the area of Bandits \citep{lattimore2020bandit}, there has been a thrust on distributed bandits with reduced number of communication rounds. Various algorithms have been proposed to minimize the regret for setups where the agents synchronize a central coordinator \citep{kanade2012distributed,hillel2013distributed,wang2019distributed,dubey2020differentially}. Recently, \cite{wang2019distributed} showed that it is possible to obtain optimal regret guarantees with number of rounds that are independent of $T$ in stochastic Multi-Armed Bandits and Linear Bandits. Further, \cite{dubey2020differentially} considered a linear bandit setup  and aimed to protect the privacy of the agents collaborating along with minimizing the number of communications rounds. Moreover, \citep{chawla2020gossiping,sankararaman2019social} considered a gossiping setup to communicate only the index of the best arm, thus reducing not only the number of communication rounds, but also the number of bits in each communication round. Similar to the bandit setup, we also attempt to find rigorous regret guarantees of the \NAM\ algorithm and a bound on the number of communication rounds.
\section{System Model}\label{sec:sys_model}
Let $[K]$ be the set of $K$ elements, or $[K] = \{1,2, \cdots, K\}$. We consider an MDP $\mathcal{M} = (\mathcal{S},\mathcal{A}, P, \bar{r})$, where $\mathcal{S} = [S]$ is the set of finite states and $\mathcal{A} = [A]$ is the set of finite actions. $P$ denotes the transition probabilities, \textit{i.e.}, on taking action $a\in\mathcal{A}$ in state $s\in\mathcal{S}$, the next state $s'\in\mathcal{S}$ follows distribution $P(\cdot|s,a)$. Also, on taking action $a\in\mathcal{A}$ in state $s\in\mathcal{S}$, an agent receives a stochastic reward $r$ drawn from a distribution over $[0,1]$ with mean $\bar{r}(s,a)$. 

We consider $M$ agents in the system, each interacting with $M$ identical environments. Let $i\in[M]$ be the indexing for agents and their corresponding environment. For an agent $i$, at time step $t$, let $s_{i,t}$ denote the state of the agent, $a_{i,t}$ be the action taken by the agent, and $r_{i,t}$ denote the reward obtained by the agent on taking action $a_{i,t}$ in state $s_{i,t}$. We assume that the $M$ environments are independent. Mathematically, $\forall~i\in[M] \text{ and }\forall~t\geq 1$, we have,
\begin{align}
    &\mathbb{P}(s_{i,t+1}|s_{1,t}, a_{1,t}, \cdots, s_{M,t}, a_{M,t},) = \mathbb{P}(s_{i,t+1}|s_{i,t}, a_{i,t})\nonumber,
\end{align}
and we take a similar assumption over the rewards $r_{i,t}$ as well. This means, the distribution of the next state $s_{i,t+1}$ and the rewards $r_{i,t}$ of agent $i\in[M]$ are conditioned only on the knowledge of the current state and action $(s_{i,t}, a_{i,t})$ of the agent $i$ and is independent of the other states and actions of the remaining agents.

Let policy $\pi:\mathcal{S}\to\mathcal{A}$ be a function to determine which action to select in state $s\in\mathcal{S}$. Note that each policy induces a Markov Chain on the states $\mathcal{S}$ with transition probabilities $P_{s,s} = P(\cdot|s,\pi(s))$. After defining a policy, we can now define the diameter of the MDP $\mathcal{M}$ as:
\begin{definition}[Diameter]
Consider the Markov Chain induced by the policy $\pi$ on the MDP $\mathcal{M}$. Let $T(s'|\mathcal{M}, \pi, s)$ be a random variable that denotes the first time step when this Markov Chain enters state $s'$ starting from state $s$. Then, the diameter of the MDP $\mathcal{M}$ is defined as:
\begin{align}
    D(\mathcal{M}) = \max_{s'\neq s}\min_{\pi}\mathbb{E}\left[T(s'|\mathcal{M}, \pi, s)\right]
\end{align}
\end{definition}
We assume that the MDP $\mathcal{M}$ has a finite diameter which means that there is a policy, such that following that policy all $s\in\mathcal{S}$ communicate with each other.

Any agent $i\in[M]$ starting from an initial random state $s_{i,1} = s$ follows an algorithm $\mathscr{A}$ till $T$ time steps to collect a cumulative reward of $R_i$. Also, let $\rho_i$ denote the average reward of the agent following algorithm $\mathscr{A}$. Or,
\begin{align}
    R_i(\mathcal{M}, \mathscr{A}, s, T) &= \sum\nolimits_{t=0}^Tr_{i,t}\\
    \rho_i(\mathcal{M}, \mathscr{A}, s) &= \lim_{T\to\infty}\frac{1}{T}\sum\nolimits_{t=0}^T\mathbb{E}\left[r_{i,t}\right]
\end{align}
Let there be an algorithm $\mathscr{A}$ which always selects action according to a stationary policy $\pi:\mathcal{S}\to\mathcal{A}$. Then, we denote $\rho(\mathcal{M}, \mathscr{A},s) = \rho(\mathcal{M}, \pi, s)$. The optimal average reward does not depend on the state \citep[Section 8.3.3]{puterman1994markov}, and hence for the optimal policy $\pi^*$ which maximizes the average reward $\rho(\mathcal{M}, \pi, s)$ we have,
\begin{align}
    \rho^*(\mathcal{M}) \coloneqq \rho^*(\mathcal{M},s) \coloneqq \max\nolimits_{\pi}\rho(\mathcal{M},s)\label{eq:optimal_reward_state_independent}.
\end{align}
Further, the optimal average reward satisfies \citep[Theorem 8.4.7]{puterman1994markov},
\begin{align}
    \rho^* + v(s) = \bar{r}(s, a^*) + \sum\nolimits_{s'}P(s'|s,a^*)v(s')~\forall s\in\mathcal{S}\label{eq:gain_bias_eqn}
\end{align}
where $a^*=\pi^*(s)$, and $v(s)$ is called the bias of the state $s$ and it denotes the extra reward obtained from starting in the state $s$. Note that $v(s)$ is not unique since if $v(s)~\forall~s$ satisfy Equation \eqref{eq:gain_bias_eqn}, then so does $v(s) + c$ for all $c\in\mathbb{R}$ and hence, the bias is translation invariant.

We aim to maximize the cumulative reward collected by all the agents. Hence, we want to develop an algorithm that, starting from no knowledge about the system, learns a policy that minimizes the regret. The regret of an algorithm $\mathscr{A}$, for starting state $s$, and running for time $T$, is defined as:
\begin{align}
    \Delta(\mathcal{M}, \mathscr{A}, s, T) := \rho^*MT - \sum\nolimits_{i \in [M]} R_i(\mathcal{M}, \mathscr{A}, s, T) \nonumber%
\end{align}
We will now present our algorithm \NAM\ that uses upper confidence bounds to bound the regret $\Delta(\mathcal{M}, \NAM, s, T)$ with high probability with only $O(MAS\log_2(MT))$ communication rounds.
\section{\NAM\ Algorithm}\label{sec:algo_description}
We consider that each agent runs an instance of the \NAM\ algorithm. The \NAM\ algorithm running at agent $i$ is described in Algorithm \ref{alg:dist_UCRL}. The algorithm proceeds in epochs indexed as $k = 1, 2, \cdots$. The start of every epoch also coincides with the synchronization step where every agent communicates with the central node to share data and update policies. This also implies that the number of synchronization rounds requires by the algorithm are the same as the number of epochs for the algorithm runs. 

\begin{algorithm}[!t]
	\caption{\NAM\ at agent $i$} \label{alg:dist_UCRL}
    \begin{algorithmic}[1]
            \STATE \textbf{Input: }{$S, A, M$}
            \STATE Set parameters $P_i(s,a,s') = 0\forall (s,a,s')\in\mathcal{S}\times\mathcal{A}\times\mathcal{S}$ and $\hat{r}_i(s,a) = 0\forall (s,a)\in\mathcal{S}\times\mathcal{A}$.
            \FOR{Epochs: $k= 1, 2, \cdots$}
                \STATE Set $\nu_{i,k}(s, a) = 0\forall (s,a)\in\mathcal{S}\times\mathcal{A}$.
                \STATE $\pi_k, N_k =$ \textsc{synchronize}$(P_i, \hat{r}_i, t)$
                \WHILE{$\nu_{i,k}(s,a) < \max(1, N_k(s,a))/M\forall (s,a)$ and Synchronization not requested}
    	            \STATE Play action $a_{i,t}=\pi_k(s_{i,t})$, observe reward $r_t$ and next state $s_{i,t+1}$.
    	            \STATE Set $\nu_{i,k}(s_{i,t},a_{i,t}) = \nu_{i,k}(s_{i,t}, a_{i,t}) + 1$, $P_i(s_{i,{i,t}}, a_{i,t}, s_{t+1}) = P_i(s_{i,t}, a_{i,t}, s_{i,t+1}) + 1$, $\hat{r}_i(s_{i,t}, a_{i,t}) = \hat{r}_i(s_{i,t}, a_{i,t}) + r_t$.
    	            \STATE Set $t = t+1$.
                \ENDWHILE
                \STATE Request synchronization
            \ENDFOR
    \end{algorithmic}
\end{algorithm}
Algorithm \ref{alg:dist_UCRL} running at agent $i$, maintains two counters, $\nu_{i,k}(s,a)$ and $P_i(s,a,s')$. $\nu_{i,k}(s,a)$ counts the number of visitations to state action pair $(s,a)$ in epoch $k$, and $P_i(s,a,s')$ counts the instances when the agent moves to state $s'$ on taking action $a$ in state $s$. The agent also stores $\hat{r}_i(s,a)$ as the cumulative reward obtained in $(s,a)$. 

\begin{algorithm}[tbp]
	\caption{\textsc{synchronize} at central node} \label{alg:synch}
    \begin{algorithmic}[1]
            \STATE \textbf{Input: }{$P_i, \hat{r}_i$} from all agents $i\in[M]$, $t$.
            \FOR{$(s,a)\in \mathcal{S}\times\mathcal{A}$}
            \STATE Set $N(s,a) = \sum_{i}\sum_{s'}P_i(s,a,s')$.
            \STATE Set $\hat{p}(s,a,s') = \frac{\sum_{i}P_i(s,a,s')}{\max\{1, N(s,a)\}}$
            \STATE Set $\hat{\Bar{r}}(s,a) = \frac{\sum_{i}\hat{r}_i(s,a)}{\max\{1, N(s,a)\}}$
            \STATE Set $\Tilde{r}(s,a) = \hat{\Bar{r}}(s,a) + \sqrt{\frac{7\log(2MSAt)}{2\max\{1, N(s,a)\}}}$
            \STATE Set $d(s,a) = \sqrt{\frac{14S\log(2MAt)}{\max\{1, N(s,a)\}}}$
            \ENDFOR
            \STATE Set $\pi$ = \textsc{Extended Value Iteration}($\hat{p}, d, \Tilde{r}, \frac{1}{\sqrt{Mt}}$)
            \STATE \textbf{Return} $\pi, N$
    \end{algorithmic}
\end{algorithm}

Let $N_k(s,a) = \sum_{i=1}^M\sum_{k'=1}^{k-1}\nu_{i,k}(s,a)$ be the total number of visitations till the start of epoch $k$ for all agents. And hence, $N_k(s,a) = 0$ at $k=1$ for all $s,a\in\mathcal{S}\times\mathcal{A}$. At the start of every epoch $k\geq 1$, the agents obtains the policy for epoch $k$ and total visitations to any state action pair $N_k(s,a)$. We will denote the time step at which epoch $k$ start and agents synchronize of the $k^{th}$ time with $t_k$. At $t=1$, the algorithm synchronizes all the agents for the very first time, or $t_1 = 1$. Later, a new epoch is triggered whenever any of the agent requests for synchronization\footnote{The synchronization request can be sent to server, which will in turn pause and synchronize the entire system.}. An agent $i$ requests synchronization whenever $\nu_{i,k}(s,a)$ becomes at least $1/M$ of $N_k(s,a)$ for any state action pair. We assume that every agent is able to receive the synchronization signal instantly and stop further processing of the current epoch. The algorithm calls \textsc{synchronize} algorithm every time after a new epoch starts and updates the policy $\pi_k$ and $N_k(s,a)$ values. Every agent now selects actions according to the policy $\pi_k$ in the epoch $k$.

The \textsc{synchronize} algorithm is described in Algorithm \ref{alg:synch}. This algorithm calculates the estimates of the transition probability $\hat{p}(\cdot|s,a)$ and the mean rewards $\hat{\Bar{r}}(s,a)$ using the samples from all the $M$ agents. We now consider a set of all plausible MDPs $\mathscr{M}(t)$ that exist in the neighborhood of the estimated MDP $\widehat{\mathcal{M}} = (\mathcal{S}, \mathcal{A}, \hat{p}, \hat{\Bar{r}})$. The mean rewards $r'(s,a)$ and the transition probabilities $p'(s,a)$ for all the MDPs in the set $\mathscr{M}(t)$ satisfies:
\begin{align}
    |\hat{\Bar{r}}(s,a) - r'(s,a)| &\leq \sqrt{\frac{7\log(2MSAt)}{2\max\{1,N(s,a)\}}}\label{eq:reward_deviation}\\
    \|\hat{p}(\cdot|s,a) - p'(\cdot|s,a)\|_1 &\leq \sqrt{\frac{14S\log(2MAt)}{\max\{1,N(s,a)\}}}\label{eq:prob_deviation}
\end{align}
After obtaining $\mathscr{M}(t)$, Algorithm \ref{alg:synch} calls the \textsc{Extended Value Iteration} algorithm which then computes the optimal policy for the optimistic MDP $\Tilde{\mathcal{M}}_t$ in the set $\mathscr{M}(t)$. The optimistic MDP satisfies $\rho^*(\Tilde{\mathcal{M}}) = \sup_{\mathrm{M}\in\mathscr{M}(t)}\rho^*(\mathrm{M})$. As described in \citep{jaksch2010near}, it is not trivial to directly find the optimistic MDP in $\mathscr{M}(t)$. Hence, we consider an extended MDP $\mathcal{M}_t^+$ which is constructed with the same state space and a continuous action space $(a,q(\cdot|\cdot, a))\in\mathcal{A}\times\mathscr{P}_t$, where $\mathscr{P}_t$ is the set of transition probabilities for action $a\in\mathcal{A}$ that satisfies Equation \eqref{eq:prob_deviation}. When $\mathscr{M}(t)$ contains the true MDP $\mathcal{M}$, the diameter of the extended MDP $\mathcal{M}_t^+$ is bounded by $D$ as the policy for which all states communicate with each other for MDP $\mathcal{M}$ also ensures that all states communicate in the extended MDP $\mathcal{M}_t^+$ as well.

\begin{algorithm}[tbp]
	\caption{\textsc{Extended Value Iteration}} \label{alg:evi}
    \begin{algorithmic}[1]
            \STATE \textbf{Input: }{$\hat{p}, d, \Tilde{r}, \epsilon$}.
            \STATE Set $u_0(s,a) = 0, u_1(s) = \max_a\Tilde{r}(s,a), i = 1$.
            \STATE Set $\pi(s) = \arg\max_a\Tilde{r}(s,a)$
            \WHILE{$\max_s\{u_i(s)-u_{i-1}(s)\} - \min_s\{u_i(s)-u_{i-1}(s)\} \geq \epsilon $}
                \STATE Sort $s_1', \cdots, s_S'$ such that $u_i(s_1')\geq \cdots\geq u_i(s_S')$.
                \STATE Set $p(s_1) = \min\{1, \hat{p}(s'|s_1, a) + d(s_1,a)/2\}$
                \STATE Set $p(s_n) = \hat{p}(s'|s_n, a), n = 2, 3, \cdots, S$.
                \STATE Set $l=1$
                \WHILE{$\sum_{s}p(s)> 1$}
                    \STATE Set $p(s_l) = \max\{0, 1-\sum_{s_n\neq s_l}p(s_n)\}$
                    \STATE Set $l = l-1$
                \ENDWHILE
                \STATE $i = i+1$
                \STATE $u_i(s) = \max_a\left\{\Tilde{r}(s,a) + \sum_{s'}p(s')u_{i-1}(s')\right\}$
                \STATE $\pi(s) = \arg\max_a\left\{\Tilde{r}(s,a) + \sum_{s'}p(s')u_{i-1}(s')\right\}$
            \ENDWHILE
            \STATE \textbf{Return} $\pi$
    \end{algorithmic}
\end{algorithm}

The \textsc{Extended Value Iteration} algorithm (Algorithm \ref{alg:evi}) follows the design of the Extended Value Iteration of UCRL2 algorithm by \cite{jaksch2010near}. As described by \cite{jaksch2010near}, \textsc{Extended Value Iteration} (\textsc{EVI}) obtains a policy $\pi$ that is $\epsilon$-optimal for the extended MDP $\mathcal{M}_t^+$ and in turn the optimistic MDP $\Tilde{\mathcal{M}}$. Algorithm \ref{alg:evi} calculates the values of the states of the extended MDP $\mathcal{M}_t^+$. The extended value iteration calculates the utilities of the states and the actions that achieve this utility. Note that unlike the extended value iteration in UCRL2 algorithm, we consider Algorithm \ref{alg:evi} to be converged when we have $\max_s(u_{i+1}- u_i(s)) - \min_s(u_{i+1}-u_i(s)) \leq \epsilon = 1/\sqrt{Mt}$. This is because we now have $M$ times more samples till any time step $t$ as compared to the UCRL algorithm. The EVI, at start of epoch $k$, returns a policy $\pi_k$ that satisfies $\rho(\Tilde{\mathcal{M}}, \pi_k)\geq \rho^*(\Tilde{\mathcal{M}}_{t_k}) - 1/\sqrt{Mt_k}$.

We also note that the central controller is not necessariliy required if the agents are in a completely connected network, they can share there data with each other and run the algorithms for the central controller by themselves. Further, the completely connected assumption can also be relaxed by considering a setup where all agents forward the messages they get. This allows the broadcast of the information. Hence, the proposed algorithm can be generalized to any network structure as long as all the agents are connected via some path.
\section{Results for \NAM\ }\label{sec:main_results}
After describing the algorithm, we now bound the regret of the \NAM\ algorithm. We show that the regret bound holds with high probability. We bound the regret incurred by the \NAM\ algorithm in the form of following theorem.

\begin{theorem}\label{thm:regret_bound}
For a MDP $\mathcal{M} = ([S], [A], P, r)$ with diameter $D$, 
for any starting state $s$, the regret of the \NAM\ algorithm, running on $M$ agents for $T$ time steps, is upper bounded with probability at least $1-\frac{1}{(MT)^{5/4}}$ as:
\begin{align}
    \Delta(\mathcal{M}, \NAM, s, T) \leq \Tilde{O}(DS\sqrt{MAT})\label{eq:regret_bounds}
\end{align}
where $\Tilde{O}$ hides the poly-log terms in $M,S, A,$ and $T$.
\end{theorem}

We let $m$ be the total number synchronizations done by agents running \NAM\ algorithm till time $T$. Then, we bound $m$ deterministically in the following theorem. 

\begin{theorem}\label{thm:bound_on_episodes}
The total number of communication rounds $m$ for dist-UCRL2 up to step $T \geq SA/M$ is upper bounded as
\begin{align}
    m \leq 1 + 2MAS + MAS \log_2\left(MT\right)
\end{align}
\end{theorem}
\begin{proof}
We use the fact that when $\nu_{i,k}(s,a) \geq N_k(s,a)/M$ for some state action pair $(s,a)$ and for some agent $i$, the total visitation count $N_{k+1}(s,a)$ is at least $N_k(s,a) + \nu_{i,k}(s,a) = N_k(s,a) (1+\frac{1}{M})$. This gives an exponential growth for the total visitation count of any state action pair. Also, since the total visitation count for all state action pairs is upper bounded by $MT$, using Jensen's inequality we bound the number of epochs in logarithmic order of $MT$. A complete proof is provided in Appendix \ref{sec:comm_round_bound_proof}.
\end{proof}

We now state the lemmas required for the proof of the Theorem \ref{thm:regret_bound}.
The first three lemmas are used to handle the stochastic nature of the algorithm and environment. The first lemma provides concentration bounds on the $\ell_1$-deviation of the transition probabilities $\hat{p}(\cdot|s,a)$ for any $(s,a)$.
\begin{lemma}\label{lem:prob_deviation}
The $\ell_1$-deviation of the true distribution and the empirical distribution using $n$ samples, over the next states given the current state $s$ and action $a$ is bounded by
\begin{align}
    \mathbb{P}\left(\|\hat{p}(\cdot|s,a) -P(\cdot|s,a)\|_1\geq \epsilon\right) \leq 2^S\exp{(-\frac{n\epsilon^2}{2})}
\end{align}
\end{lemma}
\begin{proof}
The proof follows on the lines of \citep[Theorem 2.1]{weissman2003inequalities}, using the distribution as transition probabilities.
\end{proof}

\begin{lemma}[Hoeffding's Inequality, \citep{hoeffding1994probability}]\label{lem:Hoeffdings}
Let $\{X_{t}\}_{t=1}^T$ be i.i.d. random variables in [0,1]. Then, we have,
\begin{align}
    P(\sum_{t=1}^T X_{i,t}\geq \epsilon) \leq \exp{\left(-\frac{2\epsilon^2}{T}\right)}
\end{align}
\end{lemma}

The next lemma provides concentration bounds on the sum of $M$ independent Martingale sequences for length $T$.

\begin{lemma}\label{lem:independent_martingale_sum}
Let $\{X_{i,t}\}_{t=1}^T$ be a zero-mean Martingale sequence for $i = 1,\cdots, M$ adapted to filtration $\{\mathcal{F}_t\}_{t=0}^T$. Then, if $\{X_{i,t}\}_{t=1}^T$ and $\{X_{j,t}\}_{t=1}^T$ are independent for all $i\neq j$ and $|X_{i,t}|X_{i,t-1}|\leq c$ for all $i, t$, we have,
\begin{align}
    P\left(\sum_{t=1}^T\sum_{i=1}^M X_{i,t}\geq \epsilon\right) \leq \exp{\left(-\frac{2\epsilon^2}{MTc^2}\right)}
\end{align}
\end{lemma}
\begin{proof}[Proof Sketch]
We prove this lemma similar to the proof of Azuma-Hoeffding's Inquality \citep{hoeffding1994probability}. A detailed proof is provided in Appendix \ref{sec:proof_sum_martingale}.
\end{proof}

We now  bound the growth rate of the total number of visitations to state action pair $(s,a)$, $\sum_{i=1}^M\nu_{i,k}(s,a)$, in any epoch $k$. If the total visitations are large, then the agents will incur large regret from a possibly sub-optimal policy. Hence, we have the following lemma:
\begin{lemma}\label{lem:bounded_ep_length}
For any epoch $k$, we have,
\begin{align}
    \sum\nolimits_{i=1}^M\nu_{i,k}(s,a) \leq N_k(s,a) + M-1
\end{align}
\end{lemma}
\begin{proof}
Note that agent $i$ requests for synchronization, and triggers a new epoch, whenever $\nu_{i,k}(s,a) = \lceil N_k(s,a)/M\rceil\leq N_k(s,a)/M + (M-1)/M$. Summing over all the agents $i$ gives the bound.
\end{proof}

\begin{lemma}[Lemma 19 \citep{jaksch2010near}]\label{lem:sum_of_roots}
For any sequence of number $z_1\leq z_2\leq\cdots\leq z_n$ with $z_k\leq \sum_{k'=1}^{k-1}z_{k'}\eqqcolon Z_k$, we have,
\begin{align}
    \sum_{k=1}^n\frac{z_k}{Z_k} \leq (\sqrt{2}+1)Z_n
\end{align}
\end{lemma}

The last lemma states that the span of the bias of the optimal policy defined as $\max_sv(s) - \min_s v(s)$ is bounded by the diameter $D$.
\begin{lemma}[Remark 8 from \cite{jaksch2010near}]\label{lem:span_bounded}
The span of the bias $v:\mathcal{S}\to\mathbb{R}$ of the optimal policy $\pi$ for any MDP $\mathcal{M}$ is upper bounded by its diameter $D(\mathcal{M})$, or,
\begin{align}
    sp(v) = \max_s v(s) - \min_s v(s) \leq D(\mathcal{M})
\end{align}
\end{lemma}

After stating all the necessary lemmas, we are now ready to prove the regret bound of the \NAM\ algorithm or $\Delta(\mathcal{M},\NAM, s, T)$. We provide a detailed sketch here and provide the complete proof in Appendix \ref{sec:regret_proof}.
\begin{proof}[Proof Sketch of Theorem \ref{thm:regret_bound}]
We break the regret expression,
into $4$ different sources of regrets as:

\textbf{1. Regret from deviating from expected reward}: Note that regret compares the expected optimal gain $\rho^*$ with the observed rewards $r_{i,t}$. Since, $r_{i,t}$ is a random variable between $[0,1]$, the agents suffer a regret if the observed rewards are lower as compared to the mean. Hence, we use Hoeffding's inequality \citep{hoeffding1994probability} to bound the regret generated by the randomness of the observed rewards. This gives us regret bounded by $\Tilde{O}(\sqrt{MT})$.
    
\textbf{2. Regret from deviating from the expected next state}: The algorithm, when transitioning to the next state, expects a bias given the current state. However, the bias of the realized state may be different and even lower from the expected bias. Hence, we bound the deviation from the expected bias as the algorithm moves to states. Starting from state $s$, the deviation process is modelled as a zero-mean Martingale sequence of the states visited by an agent $i$. Also, we have $M$ independent agents interacting with $M$ independent environment. We use Lemma \ref{lem:independent_martingale_sum} to bound the total deviation of the realized bias and the expected bias. This gives us regret bounded by $\Tilde{O}(D\sqrt{MT})$.
    
\textbf{3. Regret from not optimizing for the true MDP}: For epoch $k$, we run an $\frac{1}{\sqrt{Mt_k}}$-optimal policy for an optimistic MDP from the set of MDPs $\mathscr{M}(t_k)$ for $\sum_{i\in[M]}\sum_{s,a}\nu_{i,k}(s,a)$. The transition probabilities and the mean rewards of MDPs in $\mathscr{M}(t_k)$ satisfy Equation \eqref{eq:prob_deviation} and Equation \eqref{eq:reward_deviation}. When the true MDP $\mathcal{M}$ lies in $\mathscr{M}(t_k)$, we bound the regret because of the incorrect MDPs by the product of the diameter of the set $\mathscr{P}_{t_k}$, the diameter $D$ of the MDP $\mathcal{M}$, and the number of visitations to any state action pair $s,a$ in epoch $k$ using Lemma \ref{lem:span_bounded}. All we need to do now is to bound the sum of the product for all epochs. We do so by using Lemma \ref{lem:bounded_ep_length} and Lemma \ref{lem:sum_of_roots}. This gives us regret bounded by $\Tilde{O}(DS\sqrt{MAT})$.
    
\textbf{4. Regret when the estimated MDP is far from the true MDP $\mathcal{M}$}: We use Lemma \ref{lem:prob_deviation} to bound the $\ell_1$ distances of the estimated transition probability and the true transition probability given any state action pair. Further, we also use Hoeffding's inequality (Lemma \ref{lem:Hoeffdings}) to bound the distance of the true mean rewards and the estimated rewards. Taking union bounds over all time steps till $T$, we obtain the bound for all possible values of $N(s,a)$. Further, taking union bounds over all state and actions provide the desired concentration bounds for all states and actions. Finally, we take the union bounds for all values of $t\geq (MT)^{1/4}$. This gives the total bound on probability that the true MDP $\mathcal{M}$ lies in $\mathscr{M}(t_k)$ as:
    \begin{align}
        \mathbb{P}\left(\mathcal{M}\notin\mathscr{M}(t)\right) \leq \frac{1}{(MT)^{5/4}} \forall~ t\geq (MT)^{5/4}.
    \end{align}
Now summing over all the regret sources, we get the regret bound of the \NAM\ algorithm.
\end{proof}

The proof of the Theorem \ref{thm:regret_bound} and Theorem \ref{thm:bound_on_episodes} suggests that the analysis could potentially be extended to various other algorithms that follow the epoch completion condition from the UCRL-2 Algorithm of \citep{jaksch2010near}.

\section{Naive Approach: \textsc{mod-UCRL2} for multiple agents}\label{sec:mod-uc}

The proposed \NAM\ algorithm does not require the agents to work sequentially, and hence the $M$ agents can truly work in parallel. For comparison of the proposed algorithm and completeness, we also consider an extension of UCRL2 algorithm by \cite{jaksch2010near} for $M$ parallel agents which we call \textsc{mod-UCRL2}. In the \textsc{mod-UCRL2} algorithm, we assume that all the agents communicate to a centralized server at each time step $t$, and the server decides the actions for the agents at each time $t$. Thus, the communication rounds for \textsc{mod-UCRL2} algorithm is $O(T)$, while the regret analysis is not apriori clear, and is the focus of this section. We also note that even though this algorithm is an easy extension of UCRL2, the analysis of regret is not straightforward, and uses the approach that was used to prove the regret guarantees of \NAM\ because of the presence of $M$ agents.

At every time step $t$, every agent $i\in[M]$ observes its state $s_{i,t}$ and sends it to the server. The server after receiving all the states, process the requests sequentially in the order $s_{1,t}, s_{2,t}, \cdots, s_{M,t}$. The central server runs an instance of the UCRL2 algorithm with state sequence $s_{1,1}, s_{2,1}, \cdots, s_{M, 1}, s_{1, 2}, \cdots,$ and the corresponding action sequence $a_{1,1}, a_{2,1}, \cdots, a_{M, 1}, a_{1, 2}, \cdots$. The UCRL2 algorithm, running at the server, proceeds in epoch with epoch $1$ starting with $(i,t) = (1,1)$. We consider an epoch $k$ contains observations for $\{(\underline{i}_k, \underline{t}_k), (\underline{i}_k, \underline{t}_k)+1, \cdots, (\bar{i}_k, \bar{t}_k)\}$ for some $\underline{i}_k, \underline{t}_k, \bar{i}_k, \bar{t}_k$. The server maintains counters $\nu_k(s,a)$ denoting the number of times a state action pair $s,a$ is visited in the epoch $k$, and counter $N_k(s,a)$ denoting the number of times a state action pair $s,a$ is visited before the start of the epoch $k$ as:
\begin{align}
    \nu_k(s,a) &= \sum\nolimits_{(i,t) = (\underline{i}_k, \underline{t}_k)}^{(\bar{i}_k, \bar{t}_k)}\bm{1}{\{s_{i,t} = s, a_{i,t} = a\}}\\
    N_k(s,a) &= \sum\nolimits_{k'=1}^{k-1}\nu_{k'}(s,a)
\end{align}
The server updates an epoch $k$ whenever $\nu_k(s,a) = \max\{1, N_k(s,a)\}$ for some $s,a$. Following the UCRL2 algorithm, the server updates the policy at the beginning of every epoch using the observations collected till the beginning of the epoch. The complete algorithm is provided in Appendix \ref{sec:mod_UCRL2_algo}.

Note that in the \textsc{mod-UCRL2} algorithm, the agents only behaves as interfaces to the $M$ independent environments, and is essentially not practical because of the sequential interface. In the following result, we formally state the result that the  regret is upper bounded by $\Tilde{O}(DS\sqrt{MAT})$. 

\begin{theorem}\label{thm:mod_ucrl2_regret_bound}
For a MDP $\mathcal{M} = ([S], [A], P, r)$ with diameter $D$, 
for any starting state $s$, the regret of the \textsc{mod-UCRL2}\ algorithm, running on $M$ agents for $T$ time steps, is upper bounded with probability at least $1-\frac{1}{(MT)^{5/4}}$ as:
\begin{align}
    \Delta(\mathcal{M}, \textsc{mod-UCRL2}, s, T) \leq \Tilde{O}(DS\sqrt{MAT})\nonumber
\end{align}
where $\Tilde{O}$ hides the poly-log terms in $M,S, A,$ and $T$.
\end{theorem}
\begin{proof}[Proof Outline]
Similar to the proof of Theorem \ref{thm:regret_bound}, we again consider the four sources of regret. Then, breaking the regret into episodes when the \textsc{mod-UCRL2} algorithm updates policy, we obtain the required bound.
The complete proof is provided in  Appendix \ref{sec:regret_proof_mod_ucrl}.
\end{proof}

\section{Evaluations}\label{sec:evaluations}
In this section, we analyze the performance of the proposed \NAM\ algorithm empirically. We test the \NAM\ algorithm in multiple environments and also vary the number of agents in all the cases to study  the regret growth with respect to $M$. Further, we also evaluate the average number of communication steps used  by the \NAM\  algorithm till time step $T$.

\begin{figure*}[!ht]
    \centering
    \subfigure[RiverSwim environment]{
        \includegraphics[trim=0.5in 0.25in 0.75in 0.55in, clip, width=0.5\textwidth]{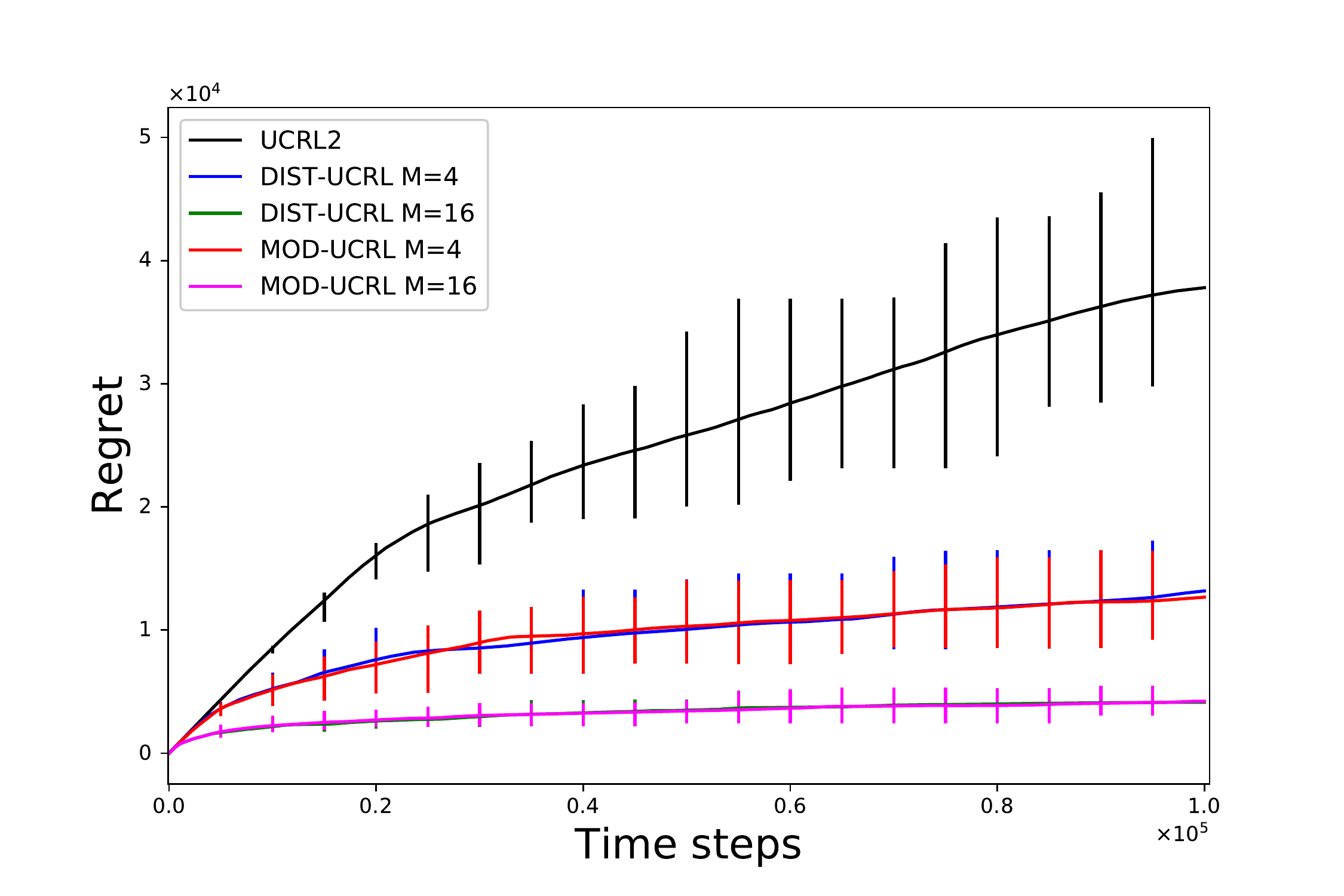}
        \label{fig:regret_rs_6}
    }
    \hspace{.1in}
    \subfigure[RiverSwim-12 environment]{
        \includegraphics[trim=0.5in 0.25in 0.75in 0.55in, clip, width=0.5\textwidth]{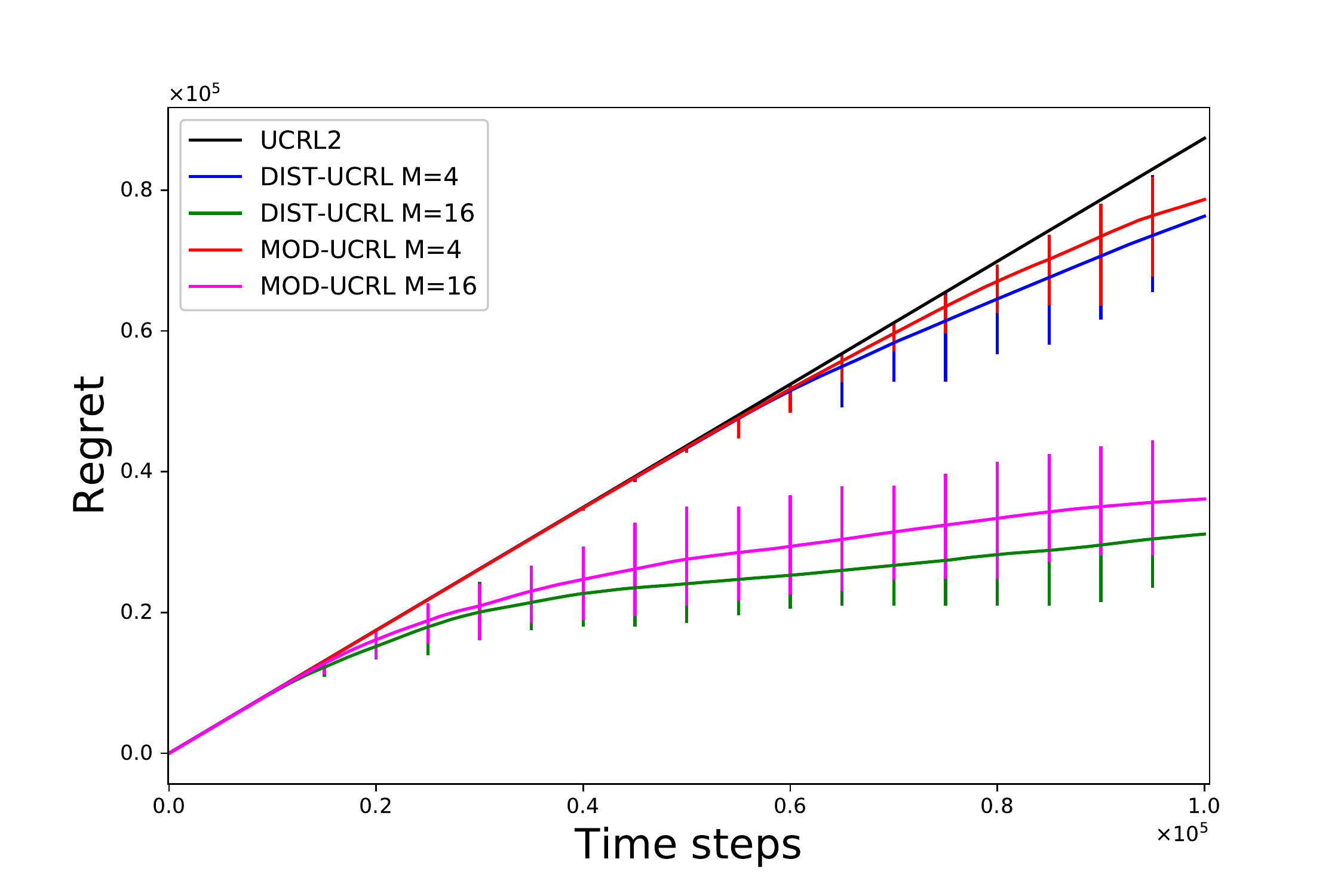}
        \label{fig:regret_rs_12}
    }
    \hspace{.1in}
    \subfigure[Gridworld environment]{
        \includegraphics[trim=0.5in 0.25in 0.75in 0.55in, clip, width=0.5\textwidth]{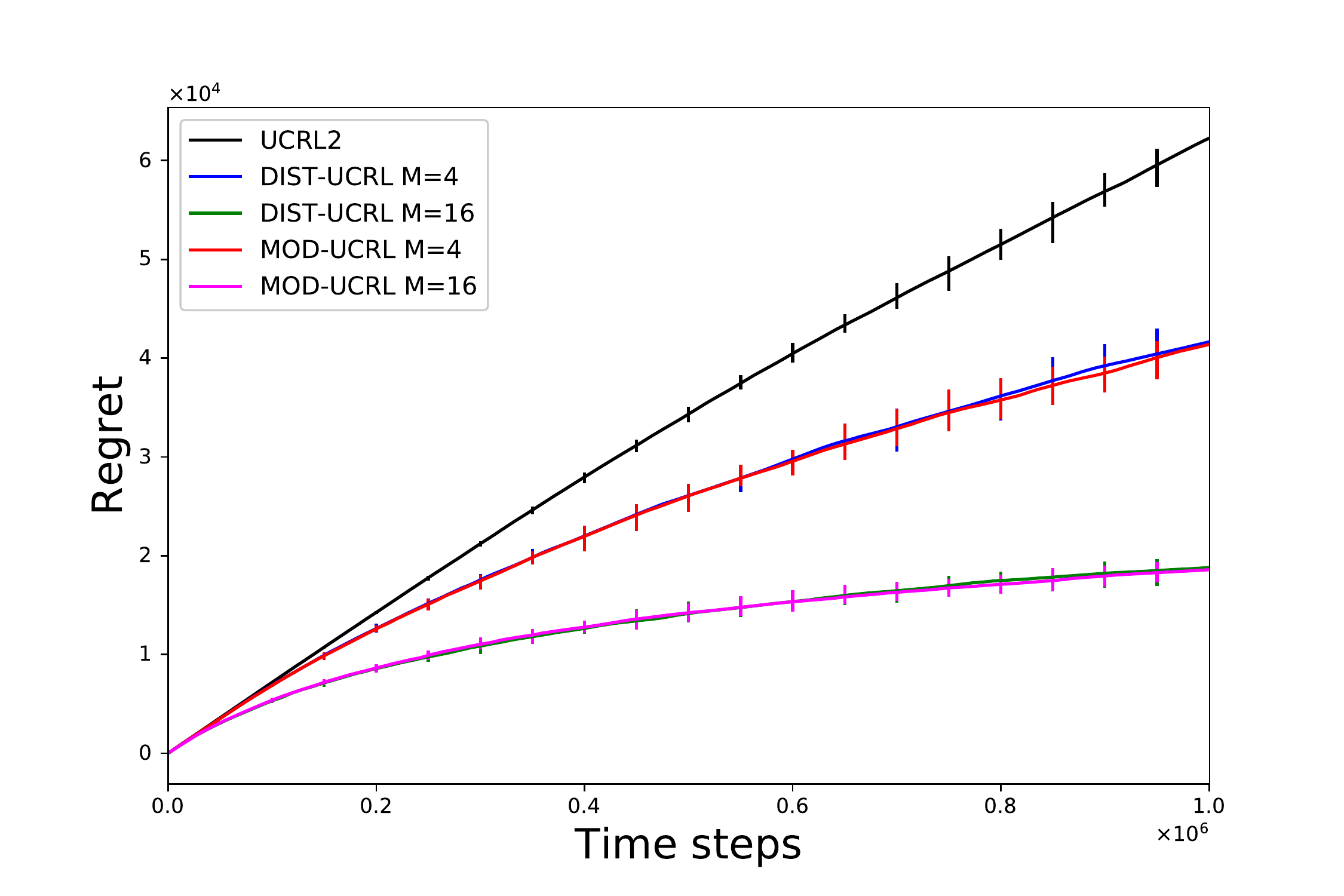}
        \label{fig:regret_gw}
    }    
    \caption{Average cumulative regret per agent under various communication strategies. The empirical regret of the \NAM\ algorithm and \textsc{mod-UCRL2} algorithm are almost identical which is expected from the regret analysis of the two algorithms.}
    \label{fig:per_agent_regret}
    \vspace{-.1in}
\end{figure*}

We first run the \NAM\ algorithm for the riverswim environment, which is a standard benchmark for model-based RL algorithm \citep{ian2013more,tossou2019near} with $6$ states and $2$ actions.
Next, we construct an extended riverswim environment with $12$ states and $2$ actions.
Finally, we use a Grid-World environment \citep{sutton2018reinforcement} for a $7\times7$ grid which amounts to $20$ states and $4$ actions.

We compare the \NAM\ algorithm against the \textsc{mod-UCRL2} algorithm. We also compare with the standard UCRL2 algorithm for $M=1$. Note that for the case of $M=1$, both, \NAM\ algorithm and the \textsc{mod-UCRL2} algorithm reduces to the UCRL2 algorithm.

\begin{figure*}[!th]
    \centering
    \subfigure[RiverSwim environment]{
        \includegraphics[trim=0.50in 0.25in 0.75in 0.55in, clip, width=0.5\textwidth]{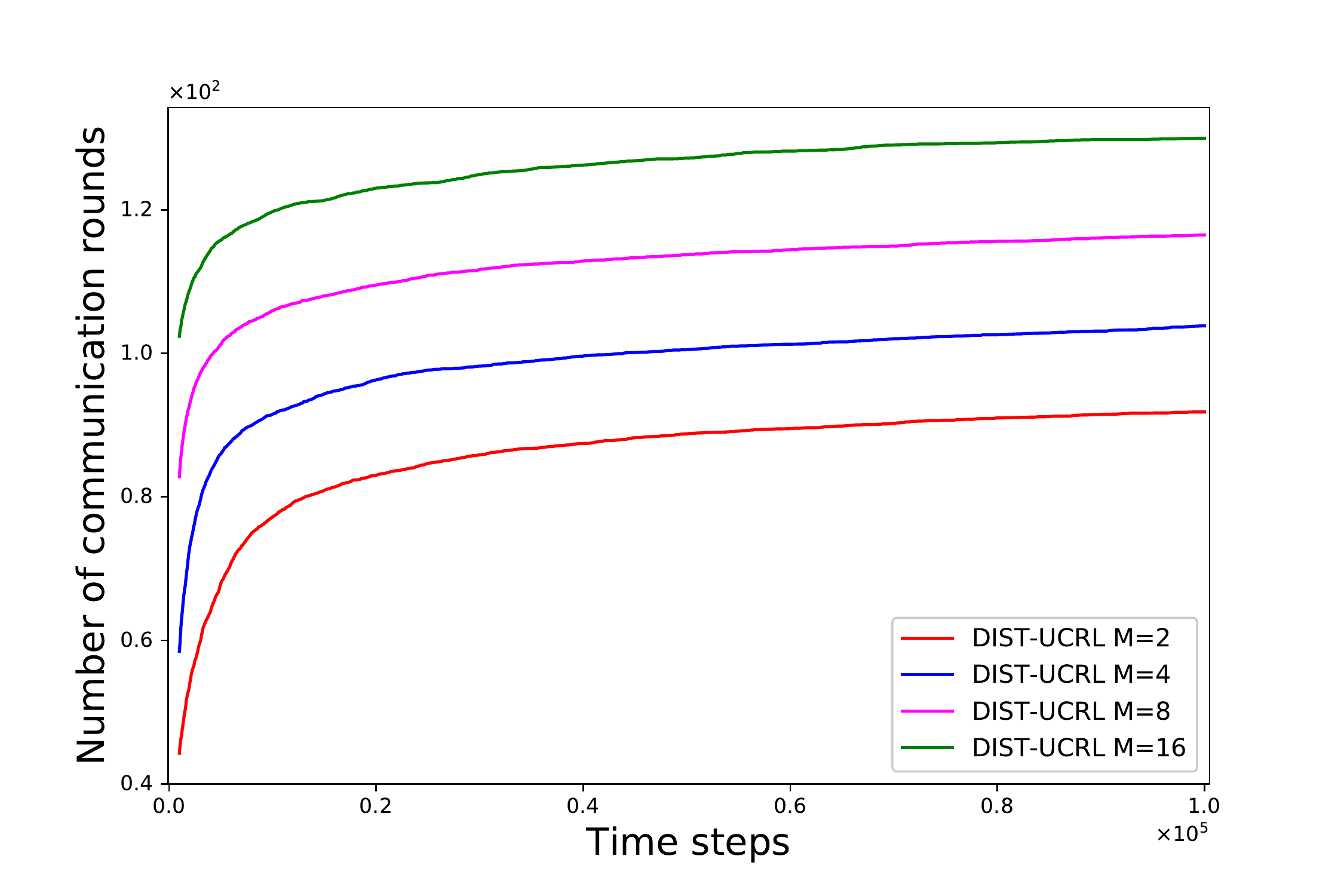}
        \label{fig:comm_rs_6}
    }
    \hspace{.1in}
    \subfigure[RiverSwim-12 environment]{
        \includegraphics[trim=0.5in 0.25in 0.75in 0.55in, clip, width=0.5\textwidth]{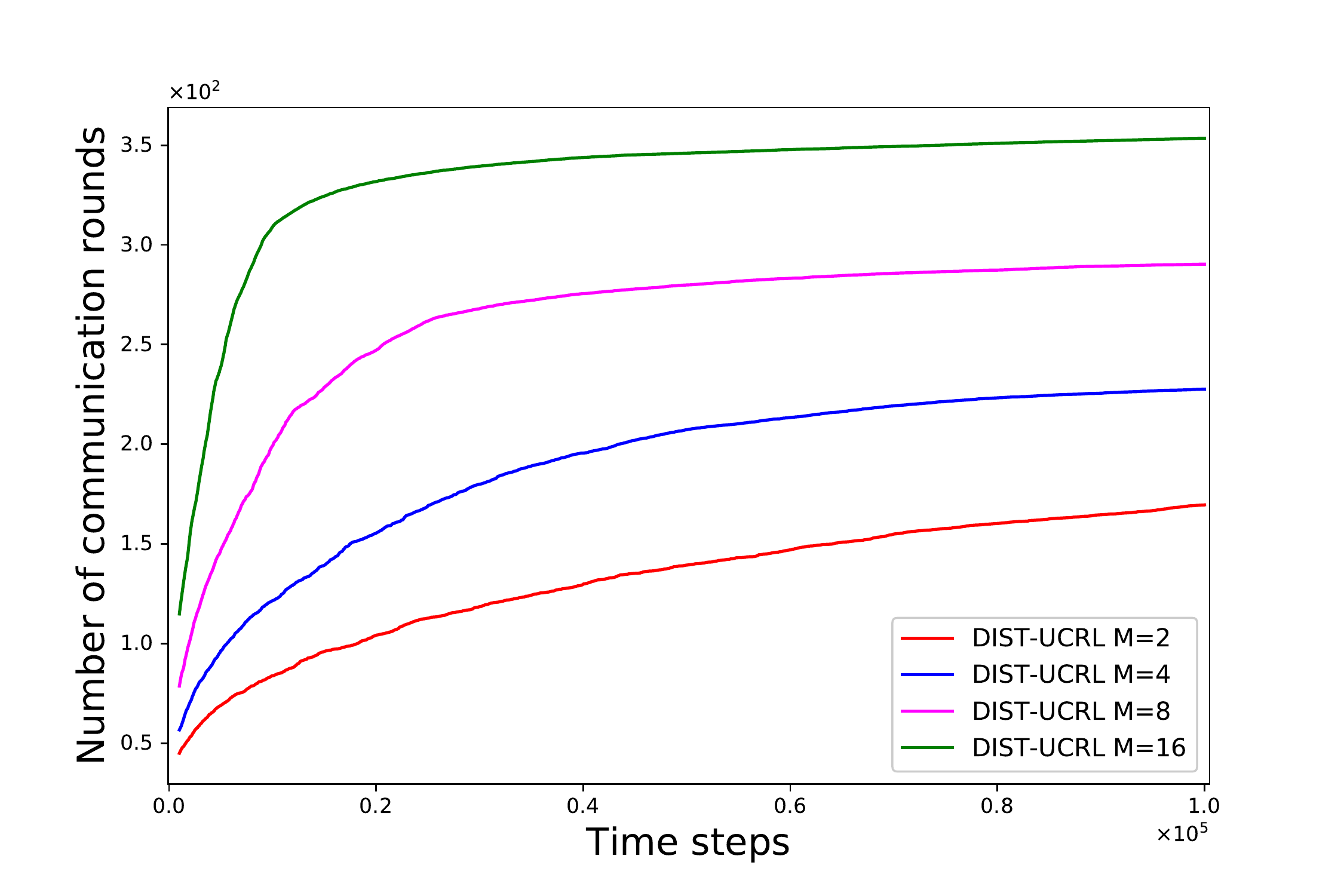}
        \label{fig:comm_rs_12}
    }
    \hspace{.1in}
    \subfigure[Gridworld environment]{
        \includegraphics[trim=0.5in 0.25in 0.75in 0.55in, clip, width=0.5\textwidth]{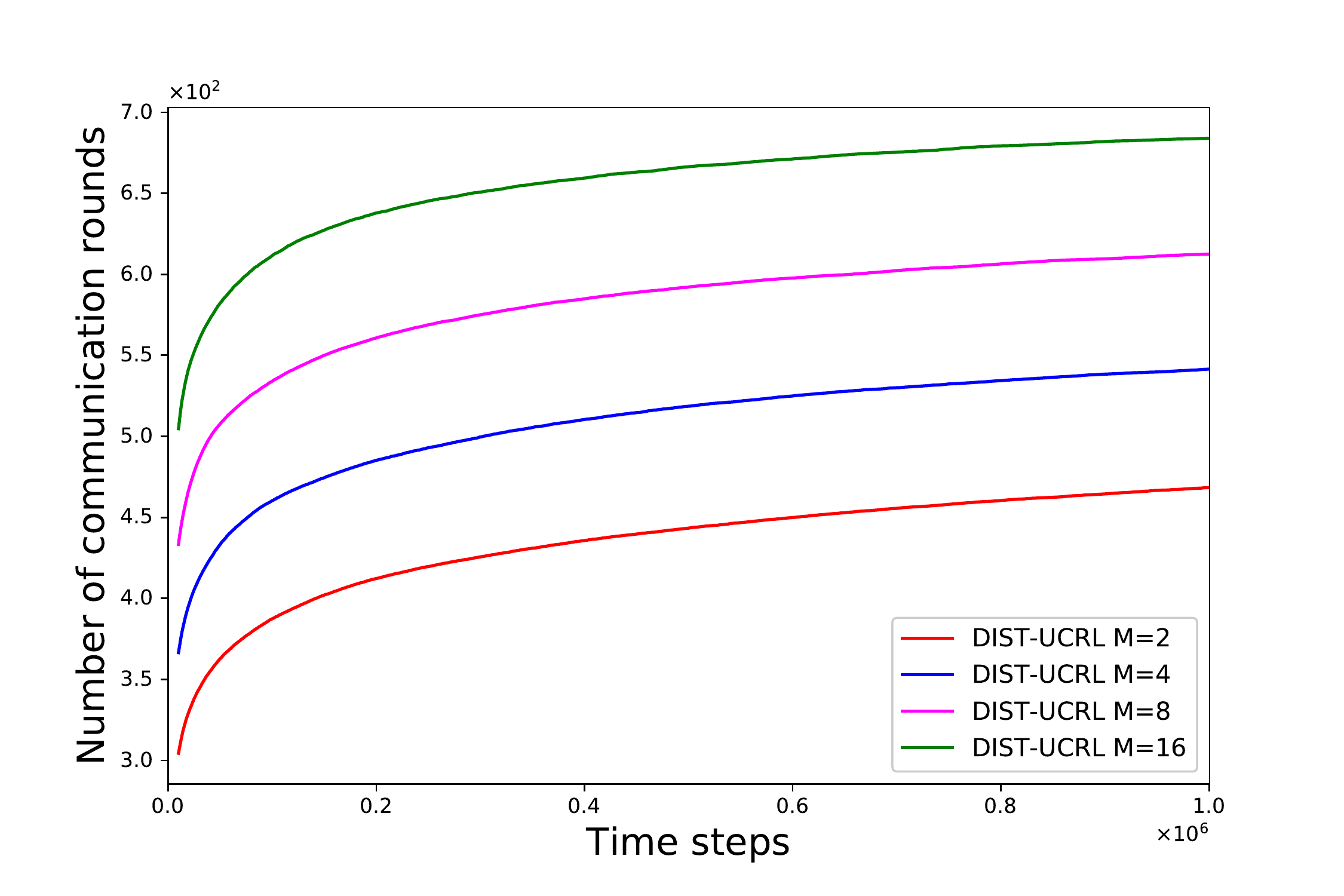}
        \label{fig:comm_gw}
    }    
    \caption{Total number of communication rounds required for the \NAM\ algorithm for multiple number agents across various environments.}
    \label{fig:total_communications}
    \vspace{-.1in}
\end{figure*}

We run $50$ independent iterations of the algorithm. We plot the average per-agent regret, $\Delta(\mathcal{M}, \NAM, s, T)/M$, over the $50$ iterations and the error bars  for both \NAM\ and \textsc{mod-UCRL2}  algorithm in Figure \ref{fig:per_agent_regret}. We vary the number of agents $M$ as $1, 4$, and $16$ to reduce clutter in figures. From Figure \ref{fig:regret_rs_6} and Figure \ref{fig:regret_gw}, we note that the per-agent regret decreases approximately by a factor of $2$ for every $4$-fold increase in the number of agents. This is expected and was indeed the goal. 

Note that in Figure \ref{fig:regret_rs_12}, the per-agent regret falls from being linear in UCRL2 ($M=1$) to significantly sublinear as the number of agents were increased to $M=4$ and $M=16$. For the extended RiverSwim environment with $12$ states, the diameter is approximately $5\times10^4$. Hence, the available time horizon of $T = 10^5$ was not sufficient for UCRL2. However, distributed RL algorithms still observed a sub-linear regret by collating their knowledge. This further demonstrates the advantage of deploying multiple parallel agents in complex environments to enhance exploration.

We also empirically evaluate the number of communication rounds required by the \NAM\ algorithm. Note that the number of communications for \textsc{mod-UCRL2}  algorithm is same as the total length for which the algorithm runs. Hence, we only plot the number of communications rounds required for the \NAM\ algorithm. We vary the number of agents $M$ in multiples for $2$ and starting from $M=2$ agents. We again run the experiment for $50$ independent iterations and plot the average number of synchronization rounds till time step $t$. 
 
We plot the number of communication rounds of the \NAM\ algorithm in Figure \ref{fig:total_communications}. We observe that the number of the synchronization rounds increase very slowly with $t$. Further, as suggested from Theorem \ref{thm:bound_on_episodes}, the number of synchronization rounds are increasing in $M$. However, we note that for large values of $t$, the increase in the number of communication rounds is sub-linear. This is because for the bound on the number of communication rounds, we take pessimistic estimates on the increase on the visitation counts $N_k(s,a)$ where only one agent $i$ updates total $N_k(s,a)$ by a factor of $(1 + 1/M)$ only when the agent $i$ triggers the synchronization round for $s,a$. However, in practice, $N_k(s,a)$ grows faster than $(1+1/M)$, as the synchronization rounds triggered for other state action pairs and other agents also contribute towards $N_k(s,a)$.
\section{Conclusion}\label{sec:concl}
In this work, we considered the problem of simultaneously reducing the cumulative regret and the number of communication rounds between $M$ agents. The $M$ agents interact with $M$ independent and identical Markov Decision Processes and share their data to learn the optimal policy faster. To this end, we  proposed \NAM, an epoch-based algorithm, following which the agents communicate in the beginning of every epoch. The data collected from the $M$ agents allows obtaining tighter deviation bounds and hence a smaller confidence set. Further, the agents trigger an epoch only after collecting sufficient samples in every epoch. This allows us to bound the regret of the \NAM\ algorithm as $\Tilde{O}(DS\sqrt{MAT})$ and the number of communication rounds as $O(MAS\log_2(MT))$. To analyze our algorithm, we also provide a concentration inequality for $M$ independent Martingale sequence of equal length which may be of an independent interest. We also evaluate the algorithm empirically and found that the average per-agent regret decreases as $O(1/\sqrt{M})$.

For comparison, we consider an extension of the UCRL2 algorithm for $M$ parallel agents working in round-robin sequence, and denote this algorithm as \textsc{mod-UCRL2}. We show that this algorithm also achieves the same regret bound as \NAM\ algorithm, while with $O(T)$ communication rounds. The regret guarantee required the techniques developed for \NAM. The evaluation results demonstrate the effectiveness of the proposed \NAM\ algorithm in achieving the same regret bound as \textsc{mod-UCRL2}, while with lower communication. Thus, the proposed \NAM\ algorithm can be utilized in training multiple parallel power starved devices using reinforcement learning.

Possible future works for the proposed works include analysing parallel RL and \NAM\ algorithm with differential privacy, where agents may send data with additional noise to protect privacy. %

\bibliography{refs}

\begin{thebibliography}{36}
\providecommand{\natexlab}[1]{#1}
\providecommand{\url}[1]{\texttt{#1}}
\expandafter\ifx\csname urlstyle\endcsname\relax
  \providecommand{\doi}[1]{doi: #1}\else
  \providecommand{\doi}{doi: \begingroup \urlstyle{rm}\Url}\fi

\bibitem[Agrawal and Jia(2017)]{agrawal2017optimistic}
Shipra Agrawal and Randy Jia.
\newblock Optimistic posterior sampling for reinforcement learning: worst-case
  regret bounds.
\newblock In \emph{Proceedings of the 31st International Conference on Neural
  Information Processing Systems}, pages 1184--1194, 2017.

\bibitem[Al-Abbasi et~al.(2019)Al-Abbasi, Ghosh, and Aggarwal]{al2019deeppool}
Abubakr~O Al-Abbasi, Arnob Ghosh, and Vaneet Aggarwal.
\newblock Deeppool: Distributed model-free algorithm for ride-sharing using
  deep reinforcement learning.
\newblock \emph{IEEE Transactions on Intelligent Transportation Systems},
  20\penalty0 (12):\penalty0 4714--4727, 2019.

\bibitem[Assran et~al.(2019)Assran, Romoff, Ballas, Pineau, and
  Rabbat]{assran2019gossip}
Mahmoud Assran, Joshua Romoff, Nicolas Ballas, Joelle Pineau, and Michael
  Rabbat.
\newblock Gossip-based actor-learner architectures for deep reinforcement
  learning.
\newblock In \emph{Advances in Neural Information Processing Systems},
  volume~32, 2019.

\bibitem[Bartlett and Tewari(2009)]{bartlett2009regal}
Peter~L Bartlett and Ambuj Tewari.
\newblock Regal: a regularization based algorithm for reinforcement learning in
  weakly communicating mdps.
\newblock In \emph{Proceedings of the Twenty-Fifth Conference on Uncertainty in
  Artificial Intelligence}, pages 35--42, 2009.

\bibitem[Bertsekas(1995)]{bertsekas1995dynamic}
Dimitri~P Bertsekas.
\newblock \emph{Dynamic programming and optimal control}, volume~1.
\newblock 1995.

\bibitem[Chawla et~al.(2020)Chawla, Sankararaman, Ganesh, and
  Shakkottai]{chawla2020gossiping}
Ronshee Chawla, Abishek Sankararaman, Ayalvadi Ganesh, and Sanjay Shakkottai.
\newblock The gossiping insert-eliminate algorithm for multi-agent bandits.
\newblock In \emph{International Conference on Artificial Intelligence and
  Statistics}, pages 3471--3481. PMLR, 2020.

\bibitem[Clemente et~al.(2017)Clemente, Castej{\'o}n, and
  Chandra]{clemente2017efficient}
Alfredo~V Clemente, Humberto~N Castej{\'o}n, and Arjun Chandra.
\newblock Efficient parallel methods for deep reinforcement learning.
\newblock \emph{arXiv preprint arXiv:1705.04862}, 2017.

\bibitem[Dubey and Pentland(2020)]{dubey2020differentially}
Abhimanyu Dubey and AlexSandy' Pentland.
\newblock Differentially-private federated linear bandits.
\newblock \emph{Advances in Neural Information Processing Systems}, 33, 2020.

\bibitem[Espeholt et~al.(2018)Espeholt, Soyer, Munos, Simonyan, Mnih, Ward,
  Doron, Firoiu, Harley, Dunning, et~al.]{espeholt2018impala}
Lasse Espeholt, Hubert Soyer, Remi Munos, Karen Simonyan, Vlad Mnih, Tom Ward,
  Yotam Doron, Vlad Firoiu, Tim Harley, Iain Dunning, et~al.
\newblock Impala: Scalable distributed deep-rl with importance weighted
  actor-learner architectures.
\newblock In \emph{International Conference on Machine Learning}, pages
  1407--1416. PMLR, 2018.

\bibitem[Gupta et~al.(2017)Gupta, Egorov, and
  Kochenderfer]{gupta2017cooperative}
Jayesh~K Gupta, Maxim Egorov, and Mykel Kochenderfer.
\newblock Cooperative multi-agent control using deep reinforcement learning.
\newblock In \emph{International Conference on Autonomous Agents and Multiagent
  Systems}, pages 66--83. Springer, 2017.

\bibitem[Haarnoja et~al.(2018)Haarnoja, Zhou, Abbeel, and
  Levine]{haarnoja2018soft}
Tuomas Haarnoja, Aurick Zhou, Pieter Abbeel, and Sergey Levine.
\newblock Soft actor-critic: Off-policy maximum entropy deep reinforcement
  learning with a stochastic actor.
\newblock In \emph{International Conference on Machine Learning}, pages
  1861--1870. PMLR, 2018.

\bibitem[Hillel et~al.(2013)Hillel, Karnin, Koren, Lempel, and
  Somekh]{hillel2013distributed}
Eshcar Hillel, Zohar Karnin, Tomer Koren, Ronny Lempel, and Oren Somekh.
\newblock Distributed exploration in multi-armed bandits.
\newblock In \emph{Proceedings of the 26th International Conference on Neural
  Information Processing Systems-Volume 1}, pages 854--862, 2013.

\bibitem[Hoeffding(1994)]{hoeffding1994probability}
Wassily Hoeffding.
\newblock Probability inequalities for sums of bounded random variables.
\newblock In \emph{The Collected Works of Wassily Hoeffding}, pages 409--426.
  Springer, 1994.

\bibitem[Horgan et~al.(2018)Horgan, Quan, Budden, Barth-Maron, Hessel, van
  Hasselt, and Silver]{horgan2018distributed}
Dan Horgan, John Quan, David Budden, Gabriel Barth-Maron, Matteo Hessel, Hado
  van Hasselt, and David Silver.
\newblock Distributed prioritized experience replay.
\newblock In \emph{International Conference on Learning Representations}, 2018.

\bibitem[Howard(1960)]{howard1960dynamic}
Ronald~A Howard.
\newblock Dynamic programming and markov processes.
\newblock 1960.

\bibitem[Hu et~al.(2020)Hu, Niu, Carrasco, Lennox, and Arvin]{hu2020voronoi}
Junyan Hu, Hanlin Niu, Joaquin Carrasco, Barry Lennox, and Farshad Arvin.
\newblock Voronoi-based multi-robot autonomous exploration in unknown
  environments via deep reinforcement learning.
\newblock \emph{IEEE Transactions on Vehicular Technology}, 2020.

\bibitem[Ian et~al.(2013)Ian, Benjamin, and Daniel]{ian2013more}
Osband Ian, Van~Roy Benjamin, and Russo Daniel.
\newblock (more) efficient reinforcement learning via posterior sampling.
\newblock In \emph{Proceedings of the 26th International Conference on Neural
  Information Processing Systems-Volume 2}, pages 3003--3011, 2013.

\bibitem[Jaksch et~al.(2010)Jaksch, Ortner, and Auer]{jaksch2010near}
Thomas Jaksch, Ronald Ortner, and Peter Auer.
\newblock Near-optimal regret bounds for reinforcement learning.
\newblock \emph{Journal of Machine Learning Research}, 11\penalty0 (4), 2010.

\bibitem[Jin et~al.(2018)Jin, Allen-Zhu, Bubeck, and Jordan]{jin2018q}
Chi Jin, Zeyuan Allen-Zhu, Sebastien Bubeck, and Michael~I Jordan.
\newblock Is q-learning provably efficient?
\newblock In \emph{Proceedings of the 32nd International Conference on Neural
  Information Processing Systems}, pages 4868--4878, 2018.

\bibitem[Kanade et~al.(2012)Kanade, Liu, and
  Radunovi{\'c}]{kanade2012distributed}
Varun Kanade, Zhenming Liu, and Bo{\v{z}}idar Radunovi{\'c}.
\newblock Distributed non-stochastic experts.
\newblock In \emph{Proceedings of the 25th International Conference on Neural
  Information Processing Systems-Volume 1}, pages 260--268, 2012.

\bibitem[Kiran et~al.(2020)Kiran, Sobh, Talpaert, Mannion, Sallab, Yogamani,
  and P{\'e}rez]{kiran2020deep}
B~Ravi Kiran, Ibrahim Sobh, Victor Talpaert, Patrick Mannion, Ahmad A~Al
  Sallab, Senthil Yogamani, and Patrick P{\'e}rez.
\newblock Deep reinforcement learning for autonomous driving: A survey.
\newblock \emph{arXiv preprint arXiv:2002.00444}, 2020.

\bibitem[Lattimore and Szepesv{\'a}ri(2020)]{lattimore2020bandit}
Tor Lattimore and Csaba Szepesv{\'a}ri.
\newblock \emph{Bandit algorithms}.
\newblock Cambridge University Press, 2020.

\bibitem[Manchella et~al.(2021)Manchella, Umrawal, and
  Aggarwal]{manchella2021flexpool}
Kaushik Manchella, Abhishek~K Umrawal, and Vaneet Aggarwal.
\newblock Flexpool: A distributed model-free deep reinforcement learning
  algorithm for joint passengers and goods transportation.
\newblock \emph{IEEE Transactions on Intelligent Transportation Systems}, 2021.

\bibitem[Mnih et~al.(2013)Mnih, Kavukcuoglu, Silver, Graves, Antonoglou,
  Wierstra, and Riedmiller]{mnih2013playing}
Volodymyr Mnih, Koray Kavukcuoglu, David Silver, Alex Graves, Ioannis
  Antonoglou, Daan Wierstra, and Martin Riedmiller.
\newblock Playing atari with deep reinforcement learning.
\newblock \emph{arXiv preprint arXiv:1312.5602}, 2013.

\bibitem[Mnih et~al.(2016)Mnih, Badia, Mirza, Graves, Lillicrap, Harley,
  Silver, and Kavukcuoglu]{mnih2016asynchronous}
Volodymyr Mnih, Adria~Puigdomenech Badia, Mehdi Mirza, Alex Graves, Timothy
  Lillicrap, Tim Harley, David Silver, and Koray Kavukcuoglu.
\newblock Asynchronous methods for deep reinforcement learning.
\newblock In \emph{International conference on machine learning}, pages
  1928--1937. PMLR, 2016.

\bibitem[Nair et~al.(2015)Nair, Srinivasan, Blackwell, Alcicek, Fearon,
  De~Maria, Panneershelvam, Suleyman, Beattie, Petersen,
  et~al.]{nair2015massively}
Arun Nair, Praveen Srinivasan, Sam Blackwell, Cagdas Alcicek, Rory Fearon,
  Alessandro De~Maria, Vedavyas Panneershelvam, Mustafa Suleyman, Charles
  Beattie, Stig Petersen, et~al.
\newblock Massively parallel methods for deep reinforcement learning.
\newblock \emph{arXiv preprint arXiv:1507.04296}, 2015.

\bibitem[Puterman(1994)]{puterman1994markov}
Martin~L. Puterman.
\newblock \emph{Markov Decision Processes: Discrete Stochastic Dynamic
  Programming}.
\newblock John Wiley \& Sons, Inc., New York, NY, USA, 1st edition, 1994.
\newblock ISBN 0471619779.

\bibitem[Sankararaman et~al.(2019)Sankararaman, Ganesh, and
  Shakkottai]{sankararaman2019social}
Abishek Sankararaman, Ayalvadi Ganesh, and Sanjay Shakkottai.
\newblock Social learning in multi agent multi armed bandits.
\newblock \emph{Proceedings of the ACM on Measurement and Analysis of Computing
  Systems}, 3\penalty0 (3):\penalty0 1--35, 2019.

\bibitem[Sartoretti et~al.(2019)Sartoretti, Wu, Paivine, Kumar, Koenig, and
  Choset]{sartoretti2019distributed}
Guillaume Sartoretti, Yue Wu, William Paivine, TK~Satish Kumar, Sven Koenig,
  and Howie Choset.
\newblock Distributed reinforcement learning for multi-robot decentralized
  collective construction.
\newblock In \emph{Distributed autonomous robotic systems}, pages 35--49.
  Springer, 2019.

\bibitem[Schulman et~al.(2015)Schulman, Levine, Abbeel, Jordan, and
  Moritz]{schulman2015trust}
John Schulman, Sergey Levine, Pieter Abbeel, Michael Jordan, and Philipp
  Moritz.
\newblock Trust region policy optimization.
\newblock In \emph{International conference on machine learning}, pages
  1889--1897. PMLR, 2015.

\bibitem[Schulman et~al.(2017)Schulman, Wolski, Dhariwal, Radford, and
  Klimov]{schulman2017proximal}
John Schulman, Filip Wolski, Prafulla Dhariwal, Alec Radford, and Oleg Klimov.
\newblock Proximal policy optimization algorithms.
\newblock \emph{arXiv preprint arXiv:1707.06347}, 2017.

\bibitem[Sutton and Barto(2018)]{sutton2018reinforcement}
Richard~S Sutton and Andrew~G Barto.
\newblock \emph{Reinforcement learning: An introduction}.
\newblock MIT press, 2018.

\bibitem[Tossou et~al.(2019)Tossou, Basu, and Dimitrakakis]{tossou2019near}
Aristide C.~Y. Tossou, Debabrota Basu, and Christos Dimitrakakis.
\newblock Near-optimal optimistic reinforcement learning using empirical
  bernstein inequalities.
\newblock \emph{CoRR}, abs/1905.12425, 2019.

\bibitem[Wang et~al.(2019)Wang, Hu, Chen, and Wang]{wang2019distributed}
Yuanhao Wang, Jiachen Hu, Xiaoyu Chen, and Liwei Wang.
\newblock Distributed bandit learning: Near-optimal regret with efficient
  communication.
\newblock In \emph{International Conference on Learning Representations}, 2019.

\bibitem[Weissman et~al.(2003)Weissman, Ordentlich, Seroussi, Verdu, and
  Weinberger]{weissman2003inequalities}
Tsachy Weissman, Erik Ordentlich, Gadiel Seroussi, Sergio Verdu, and Marcelo~J
  Weinberger.
\newblock Inequalities for the l1 deviation of the empirical distribution.
\newblock 2003.

\bibitem[Zhang et~al.(2018)Zhang, Yang, Liu, Zhang, and Basar]{zhang2018fully}
Kaiqing Zhang, Zhuoran Yang, Han Liu, Tong Zhang, and Tamer Basar.
\newblock Fully decentralized multi-agent reinforcement learning with networked
  agents.
\newblock In \emph{International Conference on Machine Learning}, pages
  5872--5881. PMLR, 2018.

\end{thebibliography}
\onecolumn
\appendices
\section{Proof of the concentration bound of sum of independent Martingales}\label{sec:proof_sum_martingale}
\begin{lemma}
Let $\{X_{i,t}\}_{t=1}^T$ be a zero-mean Martingale sequence for $i = 1,\cdots, M$ adapted to filtration $\{\mathcal{F}_t\}_{t=0}^T$. Then, if $\{X_{i,t}\}_{t=1}^T$ and $\{X_{j,t}\}_{t=1}^T$ are independent for all $i\neq j$ and $|X_{i,t}|X_{i,t-1}|\leq c$ for all $i, t$, we have,
\begin{align}
    P(\sum_{t=1}^T\sum_{i=1}^M X_{i,t}\geq \epsilon) \leq \exp{\left(-\frac{2\epsilon^2}{MTc^2}\right)}
\end{align}
\end{lemma}
\begin{proof}
For any $s> 0$, we can write
\begin{align}
    &P\left(\sum_{t=1}^T\sum_{i=1}^M X_{i,t}\geq \epsilon\right)\\
    &= P\left((\exp{\left(s\left(\sum_{i=1}^M\sum_{t=1}^T X_{i,t}\right)\right)}\geq \exp{(s\epsilon)}\right)\\
    &\leq \frac{\mathbb{E}\left[\exp{\left(s\left(\sum_{i=1}^M\sum_{t=1}^T X_{i,t}\right)\right)}\right]}{\exp{(s\epsilon)}}\\
    &= \frac{\mathbb{E}\left[\prod_{i=1}^M\exp{\left(s\left(\sum_{t=1}^T X_{i,t}\right)\right)}\right]}{\exp{(s\epsilon)}}\\
    &= \frac{\prod_{i=1}^M\mathbb{E}\left[\exp{\left(s\left(\sum_{t=1}^T X_{i,t}\right)\right)}\right]}{\exp{(s\epsilon)}}\label{eq:use_independence}\\
    &= \frac{\prod_{i=1}^M\mathbb{E}\left[\exp{\left(s\sum_{t=1}^{T-1} X_{i,t}\right)\mathbb{E}\left[\exp{\left(sX_{i,T}\right)}|X_{i,T-1}\right]}\right]}{\exp{(s\epsilon)}}\nonumber\\
    &\leq \frac{\prod_{i=1}^M\mathbb{E}\left[\exp{\left(s\left(\sum_{t=1}^{T-1} X_{i,t}\right)\right)}\exp{(s^2c^2/8)}\right]}{\exp{(s\epsilon)}}\label{eq:use_Hoeffding_lemma}\\
    &\leq \frac{\prod_{i=1}^M\prod_{t=1}^T\exp{\left(s^2c^2/8\right)}}{\exp{(s\epsilon)}}\label{eq:iterative_Hoeffding}\\
    &= \frac{\exp{\left(MTs^2c^2/8\right)}}{\exp{(s\epsilon)}}\label{eq:to_minimize}\\
    &\leq \sup_s\frac{\exp{\left(MTs^2c^2/8\right)}}{\exp{(s\epsilon)}} = \exp{\left(-\frac{2\epsilon^2}{MTc^2}\right)}
\end{align}
Equation \eqref{eq:use_independence} follows from independence of sequences from each other. Equation \eqref{eq:use_Hoeffding_lemma} follows from Hoeffding's lemma \citep{hoeffding1994probability} on the last random variable of sequence, and Equation \eqref{eq:iterative_Hoeffding} follows from applying Hoeffding's lemma iteratively. Choosing $s$ that minimizes the expression in Equation \eqref{eq:to_minimize} we get the required result.
\end{proof}
\newpage
\section{Bounds on Probability of Event $\mathcal{M}\notin\mathscr{M}(t)$}\label{sec:prob_bound}
\begin{lemma}
The probability of the event that the set $\mathscr{M}(t)$ does not contain the true MDP $\mathcal{M}$ is bounded as:
\begin{align}
    \mathbb{P}\left(\mathcal{M}\notin\mathscr{M}(t)\right) \leq \frac{1}{15(Mt)^6}
\end{align}
\label{lem:prob_bound}
\end{lemma}
\begin{proof}
From Lemma \ref{lem:prob_deviation}, the $\ell_1$ distance of a probability distribution over $S$ events with $n$ samples is bounded as:
\begin{align}
    \mathbb{P}\left(\|P(\cdot|s,a) - \hat{P}(\cdot|s,a)\|_1\geq \epsilon\right)\leq (2^S-2)\exp{\left(-\frac{n\epsilon^2}{2}\right)}\le (2^S)\exp{\left(-\frac{n\epsilon^2}{2}\right)}
\end{align}

This, for some given number of visits $n(s,a)$ to state action pair $(s,a)$ and $\epsilon = \sqrt{\frac{2}{n(s,a)}\log(2^S20 SA(Mt)^7)}\leq \sqrt{\frac{14S}{n(s,a)}\log(2A(Mt))}$ gives,
\begin{align}
    \mathbb{P}\left(\|P(\cdot|s,a) - \hat{P}(\cdot|s,a)\|_1\geq \sqrt{\frac{14S}{n(s,a)}\log(2A(Mt))}\right)&\leq (2^S)\exp{\left(-\frac{n(s,a)}{2}\frac{2}{n(s,a)}\log(2^S20 SA(Mt)^7)\right)}\\
    &= 2^S \frac{1}{2^S 20 SA(Mt)^7} = \frac{1}{20AS(Mt)^7}
\end{align}

We sum over the all the possible values of $n(s,a)$ till $t$ time-step to bound the probability that the event $\mathcal{E}_t$ does not occur as:
\begin{align}
    \sum_{n(s,a)=1}^t \frac{1}{20SA(Mt)^7} \leq \frac{1}{20SA(Mt)^6}
\end{align}

Also, the rewards $r_t$, conditioned on $(s_t, a_t) = (s,a)$, are independent and identically distributed over $[0,1]$ with mean $\bar{r}(s,a)$. Hence, using Hoeffding's concentration bounds, we have the $\ell_1$ distance of reward estimates $\hat{r}(s,a)$ of $\bar{r}(s,a)$ with $n$ samples is bounded as:
\begin{align}
    \mathbb{P}\left(|\hat{r}(s,a) - \bar{r}(s,a)|\geq \epsilon\right)\leq 2\exp{\left(-2n\epsilon^2\right)}
\end{align}

This, for $\epsilon = \sqrt{\frac{1}{2n(s,a)}\log(120 SA(Mt)^7)}\leq \sqrt{\frac{7}{2n(s,a)}\log(2 SA(Mt))}$ gives,
\begin{align}
    \mathbb{P}\left(|\hat{r}(s,a) - \bar{r}(s,a)|\geq \epsilon\right)&\leq 2\exp{\left(-2n\frac{\log(120SA(Mt)^7)}{2n(s,a)}\right)}\\
    &\leq 2\frac{1}{120SA(Mt)^7}= \frac{1}{60SA(Mt)^7}
\end{align}

We sum over the all the possible values of $n(s,a)$ till $t$ time-step to bound the probability that $\mathcal{M}\notin \mathcal{M}(t)$ does not occur as:
\begin{align}
    \mathbb{P}\left(\|P(\cdot|s,a) -\hat{P}(\cdot|s,a) \|_1\geq \sqrt{\frac{14S}{N(s,a)}\log(2A(Mt))}\right)\leq  \sum_{n(s,a)=1}^t\frac{1}{20SA(Mt)^7} \leq \frac{1}{20SA(Mt)^6}\\
    \mathbb{P}\left(|\hat{r}(s,a) -\bar{r}(s,a)|\geq \sqrt{\frac{7}{2N(s,a)}\log(2SA(Mt))}\right)\leq  \sum_{n(s,a)=1}^t\frac{1}{60SA(Mt)^7} \leq \frac{1}{60SA(Mt)^6}
\end{align}
where $N(s,a)$ denotes the number of visitations to $s,a$ till $t$.

Finally, summing over all the $s,a$, we get,
\begin{align}
    \mathbb{P}\left(\mathcal{M}\notin \mathscr{M}(t)\right)\leq \frac{1}{15(Mt)^6}
\end{align}
\end{proof}

\section{Proof of the Regret Bound}\label{sec:regret_proof}

\begin{proof}
We want to bound $\Delta(\mathcal{M}, \NAM, s, T)$. The \NAM\ algorithm proceeds in epoch and generates a new policy for each epoch $k$ using an optimistic MDP from set $\mathscr{M}(t_k)$. For that, we construct two cases as enumerated:
\begin{enumerate}[label=(\alph*)]
    \item \textbf{Case where the true MDP $\mathcal{M}$ lies in the set $\mathscr{M}(t_k)$ for all $k$:}. In this case, we will bound the regret with the terms corresponding to parts \textbf{1},  \textbf{2}, and  \textbf{3}, of the four sources of regret as mentioned in the main text, and analyze each of them. 

From the definition of regret, we have:
\begin{align}
    \Delta(\mathcal{M}, \NAM, s, T)&= \sum_{i=1}^M\sum_{t=1}^T\rho^* - \sum_{i=1}^M\sum_{t=1}^Tr_{i,t}\nonumber\\
    &= \sum_{i=1}^M\sum_{t=1}^T(\rho^* - \mathbb{E}[r_{i,t}]) + \sum_{i=1}^M\sum_{t=1}^T\left(\mathbb{E}[r_{i,t}] - r_{i,t}\right)\nonumber\\
    &= \sum_{i=1}^M\sum_{k= 1}^m\sum_{s,a}\nu_{i,k}(s,a)\left(\rho^*- \bar{r}(s,a)\right)+ \sum_{i=1}^M\sum_{t=1}^T\left(\mathbb{E}[r_{i,t}] - r_{i,t}\right)\label{eq:exp_reward_to_mean_reward}\\
    &\leq \sum_{i=1}^M\sum_{k= 1}^m\sum_{s,a}\nu_{i,k}(s,a)\left(\Tilde{\rho}_{t_k} +  \frac{1}{\sqrt{Mt_k}} - \bar{r}(s,a)\right)+ \sum_{i=1}^M\sum_{t=1}^T\left(\mathbb{E}[r_{i,t}] - r_{i,t}\right)\label{eq:optimism}
\end{align}
Equation \eqref{eq:exp_reward_to_mean_reward} follows from the fact that at time $t$ playing action $a$ in state $s$, the expected reward $\mathbb{E}[r_{i,t}] = \bar{r}(s,a)$ for any agent $i\in[M]$. Equation \eqref{eq:optimism} follows from the fact that the \textsc{Extended Value Iteration} algorithm returns an $1/\sqrt{Mt_k}$-optimal policy for the optimistic MDP $\Tilde{\mathcal{M}_t}\in\mathscr{M}(t)$ at time $t_k$ and hence the gain $\Tilde{\rho}_{t_k}$ of the policy $\pi_k$ for $\Tilde{\mathcal{M}}$ satisfies $\Tilde{\rho}_{t_k}\geq \rho^*(\Tilde{\mathcal{M}}) - \frac{1}{\sqrt{Mt_k}}\geq \rho^*(\mathcal{M}) - \frac{1}{\sqrt{Mt_k}} = \rho^* - \frac{1}{\sqrt{Mt_k}}$. Further, the algorithm runs the policy $\pi_k$ for entire epoch $k$. We will now separate the terms in Equation \eqref{eq:optimism} to different parts as explained below. 
\\
\\
\textbf{1. Regret from deviating from expected reward:} This denotes the second summation in Equation \eqref{eq:optimism}. Since the reward $r_{i,t}\in[0,1]$, we use Hoeffding's Inequality (Lemma \ref{lem:Hoeffdings}) to bound the deviation of observed rewards from expected rewards with high probability:

\begin{align}
    \mathbb{P}\left(\sum_{i=1}^M\sum_{t=1}^T\left(\mathbb{E}[r_{i,t}] - r_{i,t}\right)\leq  \sqrt{\frac{5}{8}MT\log(8MT)}\right) &\leq \exp{\left(-\frac{2}{MT}\frac{5}{8}MT\log(8MT)\right)}\label{eq:reward_exp_reward_bound}\\
    &\leq \left(\frac{1}{4MT}\right)^{5/4}< \frac{1}{12(MT)^{5/4}}
\end{align}
This completes part \textbf{1.}

We now want to bound the first term in Equation \eqref{eq:optimism}. Recall that $t_k$ was defined as the time step whenever the epoch $k$ starts. Let $\Tilde{\mathcal{M}}_k = (\mathcal{S}, \mathcal{A}, \Tilde{P}_k, , \Tilde{r}_k)$ be the optimistic MDP in $\mathscr{M}(t)$. Then, we overload the notation and let $\Tilde{P}_k(s'|s) = \Tilde{P}(s'|s,\pi_k(s))$ denote the transition probability matrix for the Markov Chain induced by the policy $\pi_k$ on MDP $\Tilde{\mathcal{M}}$. Also, let $\Tilde{v}_k$ denote the bias vector of the policy $\pi_k$ for MDP $\Tilde{\mathcal{M}}$. Then $\Tilde{P}_k\Tilde{v}_k$ is a vector and $\Tilde{P}_k\Tilde{v}_k(s)$ denote its $s^{th}$ element. Similarly, we overload the true transition probability $P$ into a matrix $P_k$ and obtain a corresponding vector $P_kv_k$. Further, both $\bar{r}$, and $\Tilde{r}_k$ satisfy Equation \eqref{eq:reward_deviation} for mean estimated mean $\hat{\bar{r}}$. Hence, for all $s,a$, we have:
\begin{align}
|\Tilde{r}_k(s,a)- \bar{r}(s,a)| &\leq |\Tilde{r}_k(s,a)- \hat{\bar{r}}_k(s,a)|+|\hat{\bar{r}}_k(s,a)- \bar{r}(s,a)|\\
&\leq   2\sqrt{\frac{7\log(2MSAt_k)}{\max\{1, N_k(s,a)\}}} \eqqcolon d_k(s,a) \label{eq:diam_reward}
\end{align}

We now consider the gain-bias relationship from \citep{puterman1994markov} as:
\begin{align}
    &\sum_{i\in[M]}\sum_{k=1}^m\sum_{s,a}\nu_{i,k}(s,a)\left(\Tilde{\rho}_{t_k} + \frac{1}{\sqrt{Mt_k}}- \bar{r}(s,a)\right)\nonumber\\
    &= \sum_{i\in[M]}\sum_{k=1}^m\sum_{s,a}\nu_{i,k}(s,a)\left(\Tilde{\rho}_{t_k}- \Tilde{r}_k(s,a) + \left(\Tilde{r}_k(s,a)- \bar{r}(s,a)\right) + \frac{1}{\sqrt{Mt_k}}\right)\nonumber\\
    &=\sum_{i\in[M]}\sum_{k=1}^m\sum_{s,a}\nu_{i,k}(s,a)\left(\Tilde{P}_k\Tilde{v}_k(s) - \Tilde{v}_k(s)+ \left(\Tilde{r}_k(s,a)- \bar{r}(s,a)\right)+ \frac{1}{\sqrt{Mt_k}}\right)\label{eq:gain_to_bias}\\
    &= \sum_{i\in[M]}\sum_{k=1}^m\sum_{s,a}\nu_{i,k}(s,a)(P_k\Tilde{v}_k(s) - \Tilde{v}_k(s)) \nonumber\\
    &~~~+ \sum_{i\in[M]}\sum_{k=1}^m\sum_{s,a}\nu_{i,k}(s,a)\left((\Tilde{P}_k\Tilde{v}_k(s) - P_k\Tilde{v}_k(s))+ \left(\Tilde{r}_k(s,a)- \bar{r}(s,a)\right) + \frac{1}{\sqrt{Mt_k}}\right)\label{eq:regret_split}
\end{align}
where Equation \eqref{eq:gain_to_bias} follows from Equation \eqref{eq:gain_bias_eqn}. Note that the first term in Equation \eqref{eq:regret_split} denotes part \textbf{2}, or the regret from deviating from the expected next state. The second term in Equation \eqref{eq:regret_split} denotes part \textbf{3}, or the regret when algorithm is not accurately optimizing for the true MDP.
\\
\\
\textbf{2. Regret from deviating from the expected next state:} Note that $t_{k+1}-1$ is the last time step of the epoch. Using this, we can upper bound the first term of the Equation \eqref{eq:regret_split} as:
\begin{align}
    \sum_{i\in[M]}\sum_{k=1}^m\sum_{s,a}\nu_{i,k}(s,a)(P\Tilde{v}_k(s) - \Tilde{v}_k(s))&= \sum_{i\in[M]}\sum_{k=1}^m\sum_{t=t_k}^{t_{k+1}-1}(P\Tilde{v}_k(s_{i,t}) - \Tilde{v}_k(s_{i,t}))\label{eq:return_s_to_s_t}\\
    &= \sum_{i\in[M]}\sum_{k=1}^m\sum_{t=t_k}^{t_{k+1}-1}(P\Tilde{v}_k(s_{i,t}) - \Tilde{v}_k(s_{i,t+1}))\nonumber\\
    &~~+\sum_{i\in[M]}\sum_{k=1}^m(\Tilde{v}_k(s_{i,t_{k+1}})-\Tilde{v}_k(s_{i,t_{k}}))\\
    &= \sum_{i\in[M]}\sum_{k=1}^m\sum_{t=t_k}^{t_{k+1}-1}(P\Tilde{v}_k(s_{i,t}) - \Tilde{v}_k(s_{i,t+1}))+MmD\label{eq:span_leq_D}\\
    &= \sum_{i\in[M]}\sum_{k=1}^m\sum_{t=t_k}^{t_{k+1}-1}(X_{i,t})+ MmD\\
    &= \sum_{i\in[M]}\sum_{t=1}^T(X_{i,t})+ MmD\label{eq:sum_indep_martingale}
\end{align}  
where Equation \eqref{eq:return_s_to_s_t} follows from the fact that $\sum_{a}\nu_{i,k}(s,a) = \sum_{t=t_k}^{t_{k+1}-1}\bm{1}\{s_t = s\}$. 
Equation \eqref{eq:span_leq_D} follows from Lemma \ref{lem:span_bounded}. Now, note that $X_{i,t}$ is a Martingale difference sequence for each $i\in[M]$. Further, the processes are independent as the environments of each agent are independent.
Hence, the first term in Equation \ref{eq:sum_indep_martingale} is sum of $M$ independent Martingale sequences of length $T$ such that $0\leq \Tilde{v}_k(s_{i,t})\leq D$. Hence, using the Lemma \ref{lem:independent_martingale_sum} for $c=2D$, with probability at least $1-1/(12(MT)^{5/4})$, we get:
\begin{align}
    \sum_{i=1}^M\sum_{k= 1}^m\sum_{s,a}\nu_{i,k}(s,a)(P\Tilde{v}_k(s) - \Tilde{v}_k(s)) &\leq D\sqrt{2MT\frac{5}{4}\log(8MT)} + MmD\label{eq:independent_Martingale}
\end{align}
This completes part \textbf{2.}\\
\\
\textbf{3. Regret from not optimizing for the true MDP:} We now attempt to bound the second term in Equation \eqref{eq:regret_split} which denotes the distance of the estimated transition probability and the true transition probabilities:
\begin{align}
&\sum_{i=1}^M\sum_{k= 1}^m\sum_{s,a}\nu_{i,k}(s,a)((\Tilde{P}_k\Tilde{v}_k(s) - P\Tilde{v}_k(s))+ ({\Tilde{r}}_k(s,a) - \bar{r}(s,a))+\frac{1}{\sqrt{Mt_k}})\nonumber\\
&\leq\sum_{s,a}\sum_{k= 1}^m\sum_{i=1}^M\nu_{i,k}(s,a)((\Tilde{P}_k\Tilde{v}_k(s) - P\Tilde{v}_k(s))+ d_k(s,a) + \frac{1}{\sqrt{Mt_k}})\label{eq:replace_reward_diff_with_d}\\
&\leq\sum_{s,a}\sum_{k= 1}^m\sum_{i=1}^M\nu_{i,k}(s,a)\left(\|\Tilde{P}_k - P\|_1\|\Tilde{v}_k(s)\|_\infty+ d_k(s,a)+ \frac{1}{\sqrt{Mt_k}}\right)\label{eq:holders}\\
&\leq \sum_{s,a}\sum_{k= 1}^m\sum_{i=1}^M\nu_{i,k}(s,a)\left(D\|\Tilde{P}_k - P\|_1+ d_k(s,a)+ \frac{1}{\sqrt{Mt_k}}\right)\label{eq:replace_span_with_D}\\
&\leq \sum_{s,a}\sum_{k= 1}^m\sum_{i=1}^M\nu_{i,k}(s,a)\left(D\|\Tilde{P}_k - P\|_1\right) + d_k(s,a)+  \sum_{s,a}\sum_{k= 1}^m\sum_{i=1}^M\nu_{i,k}(s,a)\left( \frac{1}{\sqrt{Mt_k}}\right)
\end{align}
where Equation \eqref{eq:replace_reward_diff_with_d} holds from using Equation \eqref{eq:diam_reward}. Equation \eqref{eq:holders} follows from H\"{o}lder's inequality. Equation \eqref{eq:replace_span_with_D} follows from Lemma \ref{lem:span_bounded} and noting that $\Tilde{v}$ is translation invariant, and hence we can choose $\min_s \Tilde{v}(s) \geq 0$. Now, note that $\Tilde{P}_k, P_k$, both satisfy Equation \eqref{eq:prob_deviation}. Hence, $\|\Tilde{P}_k-P_k\|_1$ is upper bounded by the diameter of the set $\mathscr{P}_{t_k}$. Further, $N_k(s,a)\leq \sum_{s,a} N_k(s,a) = Mt_k$. This gives us:

\begin{align}
&\leq \sum_{s,a}\sum_{k= 1}^m\sum_{i=1}^M\nu_{i,k}(s,a)\frac{2\sqrt{7\log(2MSAt_k)}+2D\sqrt{14S\log(2MAt_k)}+1}{\sqrt{N_k(s,a)}}\nonumber\\
&~~~+ \sum_{s,a}\sum_{k= 1}^m\sum_{i=1}^M\nu_{i,k}(s,a)\frac{1}{\sqrt{N_k(s,a)}}\\
&\leq 2\left(\sqrt{7\log(2MSAT) + 1}+D\sqrt{14S\log(2MAT)}\right)\nonumber\\
&~~~\times\left(\sum_{s,a}\sum_{k= 1}^m\frac{\sum_{i=1}^M\nu_{i,k}(s,a)}{\sqrt{N_k(s,a)}}+\sum_{s,a}\sum_{k= 1}^m\frac{\sum_{i=1}^M\nu_{i,k}(s,a)}{\sqrt{N_k(s,a)}}\right)\nonumber\\
&\leq (\sqrt{2}+1)2\left(\sqrt{7\log(2MSAT) + 1}+D\sqrt{14S\log(2MAT)}\right)\sum\nolimits_{s,a}\left(\sqrt{N(s,a)} + \sum_{k=1}^m\frac{M-1}{\sqrt{N_k(s,a)}}\right)\label{eq:bound_z}\\
&\leq (\sqrt{2}+1)2\left(\sqrt{7\log(2MSAT) + 1}+D\sqrt{14S\log(2MAT)}\right)\sum\nolimits_{s,a}\left(\sqrt{N(s,a)} + \sum_{k=1}^m(M-1)\right)\label{eq:bound_M}\\
&\leq (\sqrt{2}+1)2\left(\sqrt{7\log(2MSAT) + 1}+D\sqrt{14S\log(2MAT)}\right)\nonumber\\
&~~~\times\left(\sqrt{\left(\sum\nolimits_{s,a}1\right)\left(\sum\nolimits_{s,a}N(s,a)\right)} + m(M-1)SA\right)\label{eq:cauchy_schwarz}\\
&\leq (\sqrt{2}+1)2\left(\sqrt{7\log(2MSAT) + 1}+D\sqrt{14S\log(2MAT)}\right)\left(\sqrt{SAMT}+m(M-1)SA\right)
\end{align}
Equation \eqref{eq:bound_z} follows from Lemma \ref{lem:sum_of_roots} and Lemma \ref{lem:bounded_ep_length}. Further, we note that if $N(s,a) = 0$, then $N_k(s,a) =0$ for all $k$, and  $\nu_{i,k} = 0$ for all $k,i$. Hence, $N_k(s,a) \geq 1$ for all $k$ in Equation \eqref{eq:bound_z} and this gives Equation \eqref{eq:bound_M}. Equation \eqref{eq:cauchy_schwarz} follows from the Cauchy Schwarz inequality.

This completes part \textbf{3.}

    \item \textbf{Case where the true MDP $\mathcal\notin \mathscr{M}(t_k)$ for some $k$:} For this case, we use a trivial bound of $1$ for each agent $i$ at each time step $t$. This is because the reward $r_{i,t}$ lie in $[0,1]$ for all $i\in[M]$ and for all $t = 1, 2, \cdots, T$. Using this, we show that the regret remains bounded by $\sqrt{MT}$ with high probability . We bound the regret incurred for this case in part \textbf{4.}. 
    
\textbf{4. Regret when the estimated MDP is far from the true MDP $\mathcal{M}$:} %
 In the following, we bound the probability of the event when the $\ell_1$-deviation bounds in Equation \eqref{eq:prob_deviation} and the bounds in Equation \eqref{eq:reward_deviation} fail to hold. Note that the regret of any time step is upper bound by $1$ as $r_t\in[0,1]$. Hence,
\begin{align}
    \sum_{i=1}^M\sum_{k=1}^m\nu_{i,k}(s,a)\bm{1}_{\{\mathcal{M}\notin\mathscr{M}(t_k)\}}&=  \sum_{k=1}^m \sum_{i=1}^M\nu_{i,k}(s,a)\bm{1}_{\{\mathcal{M}\notin\mathscr{M}(t_k)\}}\\
    &\leq  \sum_{k=1}^m N_k(s,a)\bm{1}_{\{\mathcal{M}\notin\mathscr{M}(t_k)\}}\label{eq:sum_nu_k_N_k}\\
    &\leq  \sum_{k=1}^m Mt_k\bm{1}_{\{\mathcal{M}\notin\mathscr{M}(t_k)\}}\\
    &\leq  \sum_{t=1}^T Mt \bm{1}_{\{\mathcal{M}\notin\mathscr{M}(t)\}}\\
    &\leq  M\sum\nolimits_{t=1}^{(T/M)^{1/4}}t + \sum\nolimits_{t=(T/M)^{1/4}+1}^Tt\bm{1}_{\{\mathcal{M}\notin\mathscr{M}(t)\}}\nonumber\\
    &\leq  \sqrt{MT} + \sum\nolimits_{t=(T/M)^{1/4}+1}^Tt\bm{1}_{\{\mathcal{M}\notin\mathscr{M}(t)\}}
\end{align}
where \eqref{eq:sum_nu_k_N_k} follows from Lemma \ref{lem:bounded_ep_length}. Now, we bound the probability of the event $\{\mathcal{M}\notin\mathscr{M}(t)\}$ with $1/(15(Mt)^6)$ for all $t$ using Lemma \ref{lem:prob_bound} (which uses Lemma \ref{lem:prob_deviation}, refer Appendix \ref{sec:prob_bound} for a detailed proof). Taking union bounds over $\{\mathcal{M}\notin\mathscr{M}(t_k)\}, t\geq (T/M)^{1/4}+1$, we get
\begin{align}
    \sum_{t=(T/M)^{1/4}+1}^T\mathbb{P}(\bm{1}_{\{\mathcal{M}\notin\mathscr{M}(t)\}}) &\leq \sum_{t=(T/M)^{1/4}+1}^\infty \mathbb{P}(\bm{1}_{\{\mathcal{M}\notin\mathscr{M}(t)\}})\nonumber\\
    &\leq\int_{t=(T/M)^{1/4}}^\infty \mathbb{P}(\bm{1}_{\{\mathcal{M}\notin\mathscr{M}(t)\}})\nonumber\\
    &\leq \int_{t=(T/M)^{1/4}}^\infty \frac{1}{(Mt)^6}\leq \frac{M^{5/4}}{12M^6T^{5/4}}\leq \frac{1}{12(MT)^{5/4}}\nonumber
\end{align}
This complete part \textbf{4} of the regret sources and all the cases pertaining to the true MDP $\mathcal{M}$.
\end{enumerate}

Summing over all the possible sources of regret, we obtain the required bound on the regret. Further, using union bound on all the event when the concentration bounds fails to hold, we establish that the regret bound in Equation \eqref{eq:regret_bounds} holds with high probability.
\end{proof}

\newpage
\section{Bound on number of communication rounds} \label{sec:comm_round_bound_proof}
\begin{proof}
Let $N(s,a)$ be the total number of visits to the state action pair $(s,a)$ across all agents when algorithm \NAM\ terminates. Let $K(s,a)$ denote the synchronization rounds requested when $\nu_{i,k}\geq N_k(s,a)/M$ for any agent $i\in[M]$. Now, assuming that agent $i$ triggers the communication round because $\nu_{i,k}(s,a)\geq N_{k}(s,a)/M$ for some $(s,a)$ and $N_k(s,a) > 0$. Then for this $(s,a)$ we have, $N_{k+1}(s,a) \geq N_k(s,a) + \nu_{i,k}(s,a) \geq N_k(s,a)(1+1/M)$. Then, for $K(s,a) \ge 1$, we have:
\begin{align}
        N(s,a) &= \sum_{k=1}^m\sum_{i=1}^M\nu_{i,k}(s,a)\\
        &\geq \left(\frac{1}{M} + \sum_{k:\nu_{i,k}(s,a) \geq N_k(s,a)/M \text{ for any } i}\frac{N_{k}(s,a)}{M}\right)\nonumber\\
        &\geq \frac{1}{M}\left(1 + \sum_{j}N_{j}(s,a)\right)\label{eq:count_only_k_from_prev}\\
        &\geq \frac{1}{M}\left(1 + \sum_{j=1}^{K(s,a)}\left(1+\frac{1}{M}\right)^{j-1}\right)\label{eq:lower_bound_N_j}\\
        &\geq \left(1+\frac{1}{M}\right)^{K(s,a)}\label{eq:lower_bound_N_j2}
\end{align}
where Equation \eqref{eq:count_only_k_from_prev} follows for $j\in\{k:\nu_{i,k}(s,a) \geq N_k(s,a)/M \text{ for any } i\}$. Equation \eqref{eq:lower_bound_N_j} follows from using $K(s,a)$ lowest $j\in\{k:\nu_{i,k}(s,a) \geq N_k(s,a)/M \text{ for any } i\}\cup[m]$ and lower bounding the sequence $N_j(s,1)$ with $N_{j}(s,a) = 1$ for $j=1$. 

Now, we note \eqref{eq:lower_bound_N_j2} holds when $N(s,a)\ge 1$. Further, we note that whenever $N(s,a) = 0$ or $N(s,a) = 1$, we have $K(s,a) = 0$. So, to include the case of $N(s,a)=0$, we have
\begin{align}
    N(s,a) \geq (1+ \frac{1}{M})^{K(s,a)} - 1
\end{align}
Also, note that the total visitations $N(s,a)$ for all state action pairs is the total interactions of $M$ agents with their correspoinding environments which is $MT$, hence we have
\begin{align}
    MT &= \sum_{s,a}N(s,a)\\
       &\geq \sum_{s,a} \left[\left(1+\frac{1}{M}\right)^{K(s,a)} - 1\right]\\
       &\geq AS\left(1+\frac{1}{M}\right)^{\sum_{s,a}K(s,a)/AS} - SA\label{eq:AM-GM}\\
\implies\sum_{s,a}K(s,a)&\leq \frac{AS}{\log_2\left(1+\frac{1}{M}\right)} \log_2 \left(\frac{MT}{SA} + 1\right)\\
&\leq MAS\log_2\left(\frac{MT}{SA} + 1\right)\label{eq:log_xltx}
\end{align}
where Equation \eqref{eq:AM-GM} follows from Jensen's inequality. Equation \eqref{eq:log_xltx} follows from the fact that $\frac{1}{M}\leq \log_2(1+\frac{1}{M})\forall M\geq 1$. \\
Now, a new epoch is triggered when $\nu_{i,k}(s,a)\ge \max\{1, N_k(s,a)\}/M$. 
This means that a new epoch will be triggered whenever $N_k(s,a) = 0$ and any agent visits $(s,a)$ for the very first time. Hence, the algorithm can trigger $SA$ epochs even when $K(s,a) = 0$ for all $(s,a)$. Also, there is epoch $1$ trivially starts when the algorithm starts. This gives: 
\begin{align}
    m &\leq 1 + AS + \sum\nolimits_{s,a}K(s,a)\\
      &\leq 1 + 2MAS + MAS\log_2 \left(\frac{MT}{SA}\right),
\end{align}
for $T\geq SA/M$. This completes the proof.
\end{proof}
\newpage
\section{\textsc{mod-UCRL2} algorithm description }\label{sec:mod_UCRL2_algo}
\begin{algorithm}[h]
	\caption{Modified UCRL 2 at central server}
    \begin{algorithmic}[1]
            \STATE \textbf{Input: }{$S, A, M$}
            \STATE Set parameters $P(s,a,s') = 0\forall (s,a,s')\in\mathcal{S}\times\mathcal{A}\times\mathcal{S}$ and $\hat{r}(s,a) = 0, N(s,a) = 0~~\forall~(s,a)\in\mathcal{S}\times\mathcal{A}$.
            \FOR{Epochs: $k= 1, 2, \cdots$}
                \STATE Set $\nu_{k}(s, a) = 0\forall (s,a)\in\mathcal{S}\times\mathcal{A}$.
                \FOR{$(s,a)\in \mathcal{S}\times\mathcal{A}$}
                \STATE Set $N(s,a) = \sum_{s'}P(s,a,s')$.
                \STATE Set $\hat{p}(s,a,s') = \frac{P(s,a,s')}{\max\{1, N(s,a)\}}$
                \STATE Set $\hat{\Bar{r}}(s,a) = \frac{\hat{r}(s,a)}{\max\{1, N(s,a)\}}$
                \STATE Set $\Tilde{r}(s,a) = \hat{\Bar{r}}(s,a) + \sqrt{\frac{7\log(2MSAt)}{2\max\{1, N(s,a)\}}}$
                \STATE Set $d(s,a) = \sqrt{\frac{14S\log(2MAt)}{\max\{1, N(s,a)\}}}$
                \ENDFOR
                \STATE Set $\pi$ = \textsc{Extended Value Iteration}($\hat{p}, d, \Tilde{r}, \frac{1}{\sqrt{Mt}}$)
                \WHILE{$\nu_{i,k}(s,a) < \max\{1, N_k(s,a)\}\forall (s,a)$}
                        \STATE Receive state $s_{i,t}$ from agent $i$
    	                \STATE Send action $a_{i,t}=\pi_k(s_{i,t})$ to agent $i$ and receive back reward $r_t$ and next state $s_{i,t+1}$.
    	                \STATE Set $$\nu_{k}(s_{i,t},a_{i,t}) = \nu_{k}(s_{i,t}, a_{i,t}) + 1$$$$P(s_{i,{i,t}}, a_{i,t}, s_{t+1}) = P_i(s_{i,t}, a_{i,t}, s_{i,t+1}) + 1$$$$\hat{r}(s_{i,t}, a_{i,t}) = \hat{r}(s_{i,t}, a_{i,t}) + r_t$$
    	                \STATE Set $i = i+1$.
    	                \IF{$i > M$}
    	                    \STATE $i = 1$
    	                    \STATE $t = t+1$
    	               \ENDIF
                \ENDWHILE
            \ENDFOR
    \end{algorithmic}
\end{algorithm}

\section{Proof of the Regret Bound for the Modified UCRL2 Algorithm}\label{sec:regret_proof_mod_ucrl}

\begin{proof}
We want to bound $\Delta(\mathcal{M}, \textsc{mod-UCRL2}, s, T)$. For this, we reuse most of the proof of Theorem \ref{thm:regret_bound}.

First note that the number of epochs of the \textsc{mod-UCRL2} algorithm, $m'$, is upper bounded by $1 + AS + AS\log_2(MT/SA)$. The proof mostly follows from Section \ref{sec:comm_round_bound_proof}. However, the server now performs the update whenever $\nu_k(s,a) = N_k(s,a)$. This results in growth factor of $2$ instead of $(1+1/M)$ in Equation \eqref{eq:lower_bound_N_j2}. 

We again use the fact that the EVI algorithm finds an $\epsilon$-optimal policy of the optimistic MDP in the neighborhood of the true MDP $\mathcal{M}$. Recall that, we define server time $t'$ to be a pair $(i, t)$. Let server time $t_k'$ to be the pair $\underline{i}_i,\underline{t}_k$ when the epoch $k$ starts. Then, all the definitions from the analysis of Theorem \ref{thm:regret_bound} will translate here with subscripts changed from $t$ to $t'$ and $t_k$ to $t_k'$. Further, we define $|t'| = |(i,t)| = M(t-1) + i$, as the total number of interactions the agent collectively did with their environments till the server time $t'$. We also define $t'+1$, and $t'-1$ as:
\begin{align}
    t_i' + 1 = \begin{cases}
        (i, t+1) \text{ for } t_i' = (i, t), i = K, 
        \\
        (i+1, t) \text{ for } t_i' = (i, t), i < K.
        \end{cases}\\
    t_i' - 1 = \begin{cases}
        (i-1, t) \text{ for } t_i' = (i, t), i > 1, 
        \\
        (K, t-1) \text{ for } t_i' = (i, t), i = 1.
        \end{cases}
\end{align}

The \textsc{mod-UCRL2} algorithm proceeds in epoch and generates a new policy for each epoch $k$ using an optimistic MDP from set $\mathscr{M}(t_k')$. Now, similar to the proof of Theorem \ref{thm:regret_bound}, we construct two cases as enumerated:
\begin{enumerate}[label=(\alph*)]
    \item \textbf{Case where the true MDP $\mathcal{M}$ lies in the set $\mathscr{M}(t_k')$ for all $k$:} In this case, we will bound the regret with the terms corresponding to parts \textbf{1},  \textbf{2}, and  \textbf{3}, of the four sources of regret as mentioned in the main text for Theorem \ref{thm:regret_bound}, and analyze each of them.
    
    Again, from the definition of regret, we have:
    \begin{align}
    \Delta(\mathcal{M}, \textsc{mod-UCRL2}, s, T)&= \sum_{t'=1}^{(M,T)}\rho^* - \sum_{t'=1}^{(M,T)}r_{t'}\nonumber\\
    &= \sum_{t'=1}^{(M,T)}(\rho^* - \mathbb{E}[r_{t'}]) + \sum_{t'=1}^{(M,T)}\left(\mathbb{E}[r_{t'}] - r_{t'}\right)\nonumber\\
    &= \sum_{k= 1}^m\sum_{s,a}\nu_{k}(s,a)\left(\rho^*- \bar{r}(s,a)\right)+ \sum_{t'=1}^{(M,T)}\left(\mathbb{E}[r_{t'}] - r_{t'}\right)\label{eq:mod_exp_reward_to_mean_reward}\\
    &\leq \sum_{k= 1}^m\sum_{s,a}\nu_{k}(s,a)\left(\Tilde{\rho}_{t_k} +  \frac{1}{\sqrt{|t_k'|}} - \bar{r}(s,a)\right)+ \sum_{t'=1}^{(M,T)}\left(\mathbb{E}[r_{t'}] - r_{t'}\right)\label{eq:mod_optimism}
\end{align}

Equation \eqref{eq:mod_exp_reward_to_mean_reward} follows from the fact that at step $t'$ playing action $a$ in state $s$, the expected reward $\mathbb{E}[r_{t'}] = \bar{r}(s,a)$ for any agent $i\in[M]$. Equation \eqref{eq:mod_optimism} follows from the fact that the \textsc{Extended Value Iteration} algorithm returns an $1/\sqrt{|t_k'|}$-optimal policy for the optimistic MDP $\Tilde{\mathcal{M}_{t'}}\in\mathscr{M}(t)$ at server time $t_k$ and hence the gain $\Tilde{\rho}_{t_k'}$ of the policy $\pi_k$ for $\Tilde{\mathcal{M}}$ satisfies $\Tilde{\rho}_{t_k'}\geq \rho^*(\Tilde{\mathcal{M}}) - \frac{1}{\sqrt{|t_k'|}}\geq \rho^*(\mathcal{M}) - \frac{1}{\sqrt{|t_k'|}} = \rho^* - \frac{1}{\sqrt{|t_k'|}}$. Further, the algorithm runs the policy $\pi_k$ for entire epoch $k$.

Separating terms in Equation \eqref{eq:mod_optimism}, similar to that in the proof of Theorem \ref{thm:regret_bound}, we reduce the problem into sum of the three parts as explained below. 
\\
\\
\textbf{1. Regret from deviating from expected reward:} Since the reward $r_{t'}\in[0,1]$, we first bound the second summation in Equation \eqref{eq:mod_optimism} with high probability using Hoeffding's inequality (Lemma \ref{lem:Hoeffdings}):

\begin{align}
    \mathbb{P}\left(\sum_{t'=1}^{(M,T)}\left(\mathbb{E}[r_{t'}] - r_{t'}\right)
    \leq  \sqrt{\frac{5}{8}MT\log(8MT)}\right) &= \mathbb{P}\left(\sum_{i=1}^M\sum_{t=1}^T\left(\mathbb{E}[r_{i,t}] - r_{i,t}\right)
    \leq  \sqrt{\frac{5}{8}MT\log(8MT)}\right)\\
    &\leq \exp{\left(-\frac{2}{MT}\frac{5}{8}MT\log(8MT)\right)}\label{eq:mod_reward_exp_reward_bound}\\
    &\leq \left(\frac{1}{4MT}\right)^{5/4}< \frac{1}{12(MT)^{5/4}}
\end{align}
This completes part \textbf{1.}

We now want to bound the first term in Equation \eqref{eq:mod_optimism}. Based on the assumption,  the true MDP $\mathcal{M}$ lies in the set $\mathscr{M}(t_k')$. Let $\Tilde{\mathcal{M}}_k = (\mathcal{S}, \mathcal{A}, \Tilde{P}_k, , \Tilde{r})$ be the optimistic MDP in $\mathscr{M}(t')$. Then, we overload the notation and let $\Tilde{P}_k(s'|s) = \Tilde{P}(s'|s,\pi_k(s))$ denote the transition probability matrix for the Markov Chain induced by the policy $\pi_k$ on MDP $\Tilde{\mathcal{M}}$. Also, let $\Tilde{v}_k$ denote the bias vector of the policy $\pi_k$ for MDP $\Tilde{\mathcal{M}}$. Then $\Tilde{P}_k\Tilde{v}_k$ is a vector and $\Tilde{P}_k\Tilde{v}_k(s)$ denote its $s^{th}$ element. Similarly, we overload the true transition probability $P$ into a matrix $P_k$ and obtain a corresponding vector $P_kv_k$. Also, similar to the proof of Theorem \ref{thm:regret_bound}, we have
\begin{align}
|\Tilde{r}- \bar{r}(s,a)|  &\leq 2d_k(s,a) := \sqrt{\frac{7\log(2AS|t_k'|}{2\max\{1, N_k(s,a)\}}}\label{eq:mod_reward_diam}
\end{align}

We now consider the gain-bias relationship from \citep{puterman1994markov} as:
\begin{align}
    &\sum_{k=1}^m\sum_{s,a}\nu_{k}(s,a)\left(\Tilde{\rho}_{t_k'} + \frac{1}{\sqrt{|t_k'|}}- \bar{r}(s,a)\right)\\
    &= \sum_{k=1}^m\sum_{s,a}\nu_{k}(s,a)\left(\Tilde{\rho}_{t_k'}- \Tilde{r}_k(s,a) + \left(\Tilde{r}_k(s,a) - \bar{r}(s,a)\right) + \frac{1}{\sqrt{|t_k'|}}\right)\nonumber\\
    &=\sum_{k=1}^m\sum_{s,a}\nu_{k}(s,a)\left(\Tilde{P}_k\Tilde{v}_k(s) - \Tilde{v}_k(s)+ \left(\Tilde{r}_k(s,a) - \bar{r}(s,a)\right) + \frac{1}{\sqrt{|t_k'|}}\right)\label{eq:mod_gain_to_bias}\\
    &=\sum_{k=1}^m\sum_{s,a}\nu_{k}(s,a)(P_k\Tilde{v}_k(s) - \Tilde{v}_k(s))\nonumber\\
    &~~~~+ \sum_{k=1}^m\sum_{s,a}\nu_{k}(s,a)\left((\Tilde{P}_k\Tilde{v}_k(s) - P_k\Tilde{v}_k(s)) + \left(\Tilde{r}_k(s,a) - \bar{r}(s,a)\right) + \frac{1}{\sqrt{|t_k'|}}\right) \label{eq:mod_regret_split}
\end{align}
where Equation \eqref{eq:mod_gain_to_bias} follows from Equation \eqref{eq:gain_bias_eqn}. Again, the first term in Equation \eqref{eq:mod_regret_split} denotes part \textbf{2}, or the regret from deviating from the expected next state. The second term in Equation \eqref{eq:mod_regret_split} denotes part \textbf{3}, or the regret when algorithm is not accurately optimizing for the true MDP.
\\
\\
\textbf{2. Regret from deviating from the expected next state:} Let $t_{k+1}'-1 = (\bar{i}_i, \bar{t}_k)$ be the last server time step of the epoch. Using this, we can upper bound the first term of the Equation \eqref{eq:mod_regret_split} as:
\begin{align}
    \sum_{k=1}^m\sum_{s,a}\nu_{k}(s,a)(P\Tilde{v}_k(s) - \Tilde{v}_k(s))&= \sum_{k=1}^m\sum_{t'=t_k'}^{t_{k+1}'-1}(P\Tilde{v}_k(s_{t'}) - \Tilde{v}_k(s_{t'}))\label{eq:mod_return_s_to_s_t}\\
    &= \sum_{k=1}^m\sum_{t'=t_k'}^{t_{k+1}'-1}(P\Tilde{v}_k(s_{t'}) - \Tilde{v}_k(s_{t'+1}))+\sum_{k=1}^m(\Tilde{v}_k(s_{t_{k+1}})-\Tilde{v}_k(s_{t_{k}}))\\
    &= \sum_{k=1}^m\sum_{t'=t_k'}^{t_{k+1}'-1}(P\Tilde{v}_k(s_{t'}) - \Tilde{v}_k(s_{t'+1}))+\sum_{k=1}^m(\Tilde{v}_k(s_{(\bar{i}_{k+1}, \bar{t}_{k+1})})-\Tilde{v}_k(s_{(\underline{i}_{k}, \underline{t}_{k})}))\\
    &\leq \sum_{k=1}^m\sum_{t'=t_k'}^{t_{k+1}'-1}(P\Tilde{v}_k(s_{t'}) - \Tilde{v}_k(s_{t'+1}))+mD\label{eq:mod_span_leq_D}\\
    &\leq \sum_{k=1}^m\sum_{t'=(\underline{i}_k,\underline{t}_k)}^{\bar{i}_k,\bar{t}_k}(P\Tilde{v}_k(s_{t'}) - \Tilde{v}_k(s_{t'+1}))+mD\\
    &\leq \sum_{i=1}^M\sum_{k=1}^m\sum_{t=\underline{t}_k}^{\bar{t}_k}(P\Tilde{v}_k(s_{i,t}) - \Tilde{v}_k(s_{i,t+1}))+mD\\
    &\leq \sum_{i=1}^M\sum_{k=1}^m\sum_{t=\underline{t}_k}^{\bar{t}_k}X_{i,t}+mD\\
    &\leq \sum_{i=1}^M\sum_{t=1}^{T}X_{i,t}+mD\label{eq:now_use_sum_martingale}
\end{align}  
where Equation \eqref{eq:mod_return_s_to_s_t} follows from the fact that $\sum_{a}\nu_{k}(s,a) = \sum_{t'=t_k'}^{t_{k+1}'-1}\bm{1}\{s_{t'} = s\}$. 
Equation \eqref{eq:mod_span_leq_D} follows from Lemma \ref{lem:span_bounded} and from the fact that although states $s_{(\bar{i}_{k+1}, \bar{t}_{k+1})}$ and $s_{(\underline{}{i}_{k}, \underline{}{t}_{k})}$ may belong to two different processes, the processes are identical. 
Now, we note that the first term in Equation \eqref{eq:now_use_sum_martingale} is  sum of $M$ independent Martingale sequences of length $T$ such that $0\leq \Tilde{v}_k(s_{(i,t)})\leq D$. Hence, using the Lemma \ref{lem:independent_martingale_sum} for $c=2D$, with probability at least $1-1/(12(MT)^{5/4})$, we get:
\begin{align}
    \sum_{k= 1}^m\sum_{s,a}\nu_{k}(s,a)(P\Tilde{v}_k(s) - \Tilde{v}_k(s)) &\leq D\sqrt{2MT\frac{5}{4}\log(8MT)} + MmD\label{eq:mod_independent_Martingale}
\end{align}
This completes part \textbf{2.}
\\
\\
\textbf{3. Regret from not optimizing for the true MDP:} We now attempt to bound the second term in Equation \eqref{eq:mod_regret_split} which denotes the distance of the estimated transition probability and the true transition probabilities:

\begin{align}
&\sum_{k= 1}^m\sum_{s,a}\nu_{k}(s,a)((\Tilde{P}_k\Tilde{v}_k(s) - P\Tilde{v}_k(s))+
\left(\Tilde{r}_k(s,a) - \bar{r}(s,a)\right) + \frac{1}{\sqrt{|t_k'|}})\nonumber\\
&\leq\sum_{s,a}\sum_{k= 1}^m\nu_{k}(s,a)((\Tilde{P}_k\Tilde{v}_k(s) - P\Tilde{v}_k(s))+ d_k(s,a) \frac{1}{\sqrt{|t_k'|}})\nonumber\\
&\leq\sum_{k= 1}^m\nu_{k}(s,a)\left(\|\Tilde{P}_k - P\|_1\|\Tilde{v}_k(s)\|_\infty+ d_k(s,a) + \frac{1}{\sqrt{|t_k'|}}\right)\label{eq:mod_holders}\\
&\leq \sum_{s,a}\sum_{k= 1}^m\nu_{k}(s,a)\left(D\|\Tilde{P}_k - P\|_1+ d_k(s,a) + \frac{1}{\sqrt{|t_k'|}}\right)\label{eq:mod_replace_span_with_D}\\
&\leq \sum_{s,a}\sum_{k= 1}^m\nu_{k}(s,a)\left(D\|\Tilde{P}_k - P\|_1 + d_k(s,a)\right) +  \sum_{s,a}\sum_{k= 1}^m\nu_{k}(s,a)\left( \frac{1}{\sqrt{|t_k'|}}\right)
\end{align}
where Equation \eqref{eq:mod_holders} follows from H\"{o}lder's inequality. Equation \eqref{eq:mod_replace_span_with_D} follows from Lemma \ref{lem:span_bounded} and noting that $\Tilde{v}$ is translation invariant, and hence we can choose $\min_s \Tilde{v}(s) \geq 0$. Now, note that $\Tilde{P}_k, P_k$, both satisfy Equation \eqref{eq:prob_deviation}. Hence, $\|\Tilde{P}_k-P_k\|_1$ is upper bounded by the diameter of the set $\mathscr{P}_{t_k'}$. Further, $N_k(s,a)\leq \sum_{s,a} N_k(s,a) = |t_k|$. This gives us:

\begin{align}
&\leq \sum_{s,a}\sum_{k= 1}^m\nu_{k}(s,a)\frac{2D\sqrt{14S\log(2A|t_k'|)} + 2\sqrt{7\log(2SA|t_k'|)}+1}{\sqrt{N_k(s,a)}}\\
&\leq 2\left(D\sqrt{14S\log(2MAT)}+\sqrt{7\log(2MSAT) + 1}\right)\sum_{s,a}\sum_{k= 1}^m\frac{\nu_{k}(s,a)}{\sqrt{N_k(s,a)}}\nonumber\\
&\leq (\sqrt{2}+1)\left(2D\sqrt{14S\log(2MAT)}+\sqrt{7\log(2MSAT) + 1}\right)\sum\nolimits_{s,a}\left(\sqrt{N(s,a)}\right)\label{eq:mod_bound_z}\\
&\leq (\sqrt{2}+1)\left(2D\sqrt{14S\log(2MAT)}+\sqrt{7\log(2MSAT) + 1}\right)\sum\nolimits_{s,a}\left(\sqrt{N(s,a)}\right)\label{eq:mod_bound_M}\\
&\leq (\sqrt{2}+1)\left(2D\sqrt{14S\log(2MAT)}+\sqrt{7\log(2MSAT) + 1}\right)\left(\sqrt{\left(\sum\nolimits_{s,a}1\right)\left(\sum\nolimits_{s,a}N(s,a)\right)}\right)\label{eq:mod_cauchy_schwarz}\\
&\leq (\sqrt{2}+1)\left(2D\sqrt{14S\log(2MAT)}+\sqrt{7\log(2MSAT) + 1}\right)\left(\sqrt{SAMT}\right)
\end{align}
Equation \eqref{eq:mod_bound_z} follows from Lemma \ref{lem:sum_of_roots}. Further, we note that if $N(s,a) = 0$, then $N_k(s,a) =0$ for all $k$, and  $\nu_{k} = 0$ for all $k,i$. Hence, $N_k(s,a) \geq 1$ for all $k$ in Equation \eqref{eq:mod_bound_z} and this gives Equation \eqref{eq:mod_bound_M}. Equation \eqref{eq:mod_cauchy_schwarz} follows from the Cauchy Schwarz inequality. This completes part \textbf{3.}

\item  \textbf{Case where the true MDP $\mathcal\notin \mathscr{M}(t_k')$ for some $k$:} For this case, we use a trivial bound of $1$ at each server time step $t'$. This is because the reward $r_{(i,t)}$ lie in $[0,1]$ for all $i\in[M]$ and for all $t = 1, 2, \cdots, T$. Using this, we show that the regret remains bounded by $\sqrt{MT}$ with high probability . We bound the regret incurred for this case in part \textbf{4.}

\textbf{4. Regret when the estimated MDP is far from the true MDP $\mathcal{M}$:} We bound the probability of  $\mathcal{M}\notin\mathscr{M}(t_k')$ for some $k$, as follows:%
\begin{align}
    \sum_{k=1}^m\nu_{k}(s,a)\bm{1}_{\{\mathcal{M}\notin\mathscr{M}(t_k')\}}&=  \sum_{k=1}^m \nu_{k}(s,a)\bm{1}_{\{\mathcal{M}\notin\mathscr{M}(t_k')\}}\\
    &\leq  \sum_{k=1}^m N_k(s,a)\bm{1}_{\{\mathcal{M}\notin\mathscr{M}(t_k')\}}\label{eq:mod_sum_nu_k_N_k}\\
    &\leq  \sum_{k=1}^m |t_k'|\bm{1}_{\{\mathcal{M}\notin\mathscr{M}(t_k')\}}\\
    &\leq  \sum_{|t'|=1}^{MT} |t'| \bm{1}_{\{\mathcal{M}\notin\mathscr{M}(t')\}}\\
    &\leq  \sum\nolimits_{t=1}^{(MT)^{1/4}}|t'| + \sum\nolimits_{t=(MT)^{1/4}+1}^{MT}|t'|\bm{1}_{\{\mathcal{M}\notin\mathscr{M}(t')\}}\nonumber\\
    &\leq  \sqrt{MT} + \sum\nolimits_{|t'|=(MT)^{1/4}+1}^{MT}|t'|\bm{1}_{\{\mathcal{M}\notin\mathscr{M}(t')\}}
\end{align}
where \eqref{eq:mod_sum_nu_k_N_k} follows from Lemma \ref{lem:bounded_ep_length}. Now, we bound the probability of the event $\{\mathcal{M}\notin\mathscr{M}(t')\}$ with $1/(15(|t'|)^6)$ for all $t'$ using Lemma \ref{lem:prob_bound} (which uses Lemma \ref{lem:prob_deviation}, refer Appendix \ref{sec:prob_bound} for a detailed proof). Taking union bounds over $\{\mathcal{M}\notin\mathscr{M}(t_k')\}, |t'|\geq (MT)^{1/4}+1$, we get
\begin{align}
    \sum_{|t'|=(MT)^{1/4}+1}^T\mathbb{P}(\bm{1}_{\{\mathcal{M}\notin\mathscr{M}(t')\}}) &\leq \sum_{|t'|=(MT)^{1/4}+1}^\infty \mathbb{P}(\bm{1}_{\{\mathcal{M}\notin\mathscr{M}(t')\}})\nonumber\\
    &\leq\int_{|t'|=(MT)^{1/4}}^\infty \mathbb{P}(\bm{1}_{\{\mathcal{M}\notin\mathscr{M}(t')\}})\nonumber\\
    &\leq \int_{|t'|=(MT)^{1/4}}^\infty \frac{1}{(|t'|)^6}\leq  \frac{1}{12(MT)^{5/4}}\nonumber
\end{align}
This complete part \textbf{4} of the regret sources and all the cases pertaining to the true MDP $\mathcal{M}$.
\end{enumerate}

Summing over all the possible sources of regret, we obtain the required bound on the regret. Further, using union bound on all the event when the concentration bounds fails to hold, we establish that the regret bound in Equation \eqref{eq:regret_bounds} holds with high probability.
\end{proof}

\end{document}